%% file: main.tex
\documentclass{article}

\PassOptionsToPackage{numbers, compress}{natbib}

\usepackage[final]{./tex/neurips_2020}

\usepackage[utf8]{inputenc} 
\usepackage[T1]{fontenc}    
\usepackage{hyperref}       
\usepackage{url}            
\usepackage{booktabs}       
\usepackage{amsfonts}       
\usepackage{nicefrac}       
\usepackage{microtype}      
\usepackage{enumitem}

\input{./tex/macros}

\title{On the Equivalence between Online and Private Learnability beyond Binary Classification}

\author{%
Young Hun Jung\thanks{Equal Contribution. Please note that this is a corrected version of the paper originally published in NeurIPS 2020. See Section~\ref{sec:correction}.} \\
  Department of Statistics\\
  University of Michigan\\
  Ann Arbor, MI 48109 \\
  \texttt{yhjung@umich.edu} \\
  \And
  Baekjin Kim$^*$ \\
  Department of Statistics\\
  University of Michigan\\
  Ann Arbor, MI 48109 \\
  \texttt{baekjin@umich.edu} \\
  \And 
  Ambuj Tewari \\
  Department of Statistics\\
  University of Michigan\\
  Ann Arbor, MI 48109 \\
  \texttt{tewaria@umich.edu} \\  
}

\begin{document}

\maketitle

\input{abstract}

\input{introduction}

\input{preliminaries}

\input{link}

\input{PL_OL}

\input{OL_PL_multiclass}

\input{OL_PL_regression}

\input{discussion}

\input{correction}

\section*{Broader Impact}
As this paper is purely theoretical, discussing broader impact is not applicable.

%

\begin{ack}
We acknowledge the support of NSF via grants CAREER IIS-1452099 and IIS-2007055.
\end{ack}

%
%

\bibliography{./tex/myref}

\newpage
\appendix

\input{link_app}
\input{PL_OL_app}
\input{OL_PL_app}

\end{document}

%% file: tex/macros.tex
\newcount\comments  
\newcommand{\genComments}[2]{\ifnum\comments=1{\color{#1}{#2}}\fi}

\usepackage{amsmath, amssymb, amsthm}
\usepackage{color, soul}
\usepackage{algorithm, algorithmic, multicol}
\usepackage[pdftex]{graphicx}

\newtheorem{theorem}{Theorem}
\newtheorem{definition}[theorem]{Definition}
\newtheorem{lemma}[theorem]{Lemma}
\newtheorem{corollary}[theorem]{Corollary}
\newtheorem{proposition}[theorem]{Proposition}

\DeclareMathOperator*{\argmax}{arg\,max}

\newcommand{\Ldim}{\mathrm{Ldim}}
\newcommand{\Pdim}{\mathrm{Pdim}}
\newcommand{\fat}{\mathrm{fat}}
\newcommand{\twr}{\mathrm{twr}}
\newcommand{\ord}{\mathrm{ord}}
\newcommand{\CNC}{\textsc{ColorAndChoose}}
\newcommand{\SOAt}{\textsc{SOA}$_{\tau}$}
\newcommand{\SOA}{\textsc{SOA}}

\newcommand{\prob}{\mathbb{P}}
\newcommand{\E}{\mathbb{E}}

\newcommand{\cX}{\mathcal{X}}
\newcommand{\cY}{\mathcal{Y}}
\newcommand{\cH}{\mathcal{H}}
\newcommand{\cHprime}{\mathcal{H}^{\prime}}
\newcommand{\cF}{\mathcal{F}}
\newcommand{\cP}{\mathcal{P}}
\newcommand{\cA}{\mathcal{A}}
\newcommand{\cO}{\mathcal{O}}

\newcommand{\cD}{\mathcal{D}}

\newcommand{\istar}{i^{*}}
\newcommand{\logstar}{\log^{*}}

\newcommand{\kprime}{k^{\prime}}
\newcommand{\uprime}{u^{\prime}}
\newcommand{\Tprime}{T^{\prime}}
\newcommand{\Dprime}{D^{\prime}}
\newcommand{\Sprime}{S^{\prime}}
\newcommand{\Tprimeprime}{T^{\prime\prime}}
\newcommand{\ubar}{\bar{u}}

\newcommand{\zeroone}{\ell^{0-1}}
\newcommand{\yhat}{\hat{y}}
\newcommand{\ind}{\mathbb{I}}

%% file: abstract.tex

\begin{abstract}
%

\citet{alon2019private} and \citet{bun2020equivalence} recently showed that online learnability and private PAC learnability are equivalent in binary classification.
We investigate whether this equivalence extends to multi-class classification and regression.
First, we show that private learnability implies online learnability in both settings. Our extension involves studying
a novel variant of the Littlestone dimension that depends on a tolerance parameter and on an appropriate generalization
of the concept of threshold functions beyond binary classification. Second, we show that while online learnability
continues to imply private learnability in multi-class classification, current proof techniques encounter significant hurdles
in the regression setting. While the equivalence for regression remains open, we provide non-trivial sufficient
conditions for an online learnable class to also be privately learnable.
\end{abstract}

%% file: introduction.tex
\section{Introduction}

\textit{Online learning} and \textit{differentially-private (DP) learning} have been well-studied in the machine learning literature. 
While these two subjects are seemingly unrelated, recent papers have revealed a strong connection 
between online and private learnability via the notion of \textit{stability}
\citep{abernethy2019online, agarwal2017price, gonen2019private}. 
The notion of differential privacy is,
at its core, less about privacy and more about algorithmic stability since the output distribution of a 
DP algorithm should be robust to small changes in the input. 
Stability also plays a key role in developing online learning algorithms such as
follow-the-perturbed-leader (FTPL) and follow-the-regularized-leader (FTRL) \citep{abernethy2014online}.

Recently \citet{alon2019private} and \citet{bun2020equivalence} showed that 
online learnability and private PAC learnability are equivalent in binary classification. 
\citet{alon2019private} showed that private PAC learnability implies finite Littlestone dimension (Ldim) in two steps;
(i) every approximately DP learner for a class with Ldim $d$ 
requires $\Omega(\logstar d)$ thresholds (see Section \ref{sec:additional.notation} for the definition of $\logstar$), and (ii) the class of thresholds over $\mathbb{N}$ 
cannot be learned in a private manner. 
\citet{bun2020equivalence} proved the converse statement via a notion of algorithmic stability, called \textit{global stability}.
They showed (i) every class with finite Ldim 
can be learned by a globally-stable learning algorithm and (ii) they use global stability to derive a DP algorithm.
In this work, we investigate whether this equivalence extends to multi-class classification (MC) 
and regression, which is one of open questions raised by \citet{bun2020equivalence}.

In general, online learning and private learning for MC and regression
have been less studied. In binary classification
without considering privacy, the Vapnik-Chervonenkis dimension (VCdim) of hypothesis classes 
yields tight sample complexity bounds in the batch learning setting, 
and \citet{littlestone1988learning} defined Ldim as a combinatorial parameter that was later shown to fully characterize
hypothesis classes that are learnable in the online setting \cite{ben2009agnostic}.
Until recently, however, it was unknown what complexity measures for MC or regression classes characterize online or private learnability.
\citet{daniely2015multiclass} extended the Ldim to the MC setting, 
and \citet{rakhlin2015online} proposed the sequential fat- shattering dimension, an online counterpart of the fat-shattering dimension in the batch setting \citep{bartlett1996fat}.

\subsection{Related works}
DP has been extensively studied in the machine learning literature \citep{dwork2009differential, 
dwork2014algorithmic, sarwate2013signal}. Private PAC and agnostic learning 
were formally studied in the seminal work of \citet{kasiviswanathan2011can}, and the sample complexities of private learners were characterized in the later work of \citet{beimel2013characterizing}.

\citet{dwork2014algorithmic} identified stability as a common factor of learning and differential privacy. 
\citet{abernethy2019online} proposed a DP-inspired stability-based methodology to design 
online learning algorithms with excellent theoretical guarantees, and \citet{agarwal2017price} showed that
stabilization techniques such as regularization or perturbation in online learning preserve DP.
\citet{feldman2014sample} relied on communication complexity to show that every purely DP
learnable class has a finite Ldim.
Purely DP learnability is a stronger condition than online learnability, 
which means that there exist online learnable classes that are not purely DP learnable.
More recently, \citet{alon2019private} and \citet{bun2020equivalence} established the equivalence between 
online and private learnability in a non-constructive manner. 
\citet{gonen2019private} derived an efficient black-box reduction from purely DP learning to online learning.
In the paper we will focus on approximate DP instead of pure DP (see Definition~\ref{def:DP}).

\subsection{Main results and techniques}

Our main technical contributions are as follows.

\begin{itemize}[leftmargin = *]
\item In Section \ref{sec:link}, we develop a novel variant of the Littlestone dimension that depends on a 
tolerance parameter $\tau$, denoted by $\Ldim_{\tau}$. 
While online learnable regression problems do not naturally reduce to learnable MC problems by discretization, 
this relaxed complexity measure bridges online MC learnability 
and regression learnability in that it allows us to consider a regression problem as a relatively simpler MC problem (see Proposition \ref{prop:Ldim.fat}).

\item In Section \ref{sec:PL.OL}, we show that private PAC learnability implies online learnability in both MC and regression settings. 
We appropriately generalize the concept of threshold functions beyond the binary classification setting 
and lower bound the number of these functions using the complexity measures 
(see Theorem \ref{thm:alon.thm3}). Then the argument of \citet{alon2019private} 
that an infinite class of thresholds cannot be privately learned can be extended to 
both settings of interest.

\item In Section \ref{sec:OLPL}, we show that while online learnability continues to imply private learnability in MC (see Theorem \ref{thm:OLPL}), current proof techniques based on \textit{global stability} and \textit{stable histogram} encounter significant obstacles in the regression problem. While this direction for regression setting still remains open,
we provide non-trivial sufficient conditions for an online learnable class to also be privately learnable (see Theorem \ref{thm:sufficient}).
\end{itemize}

%% file: preliminaries.tex
\section{Preliminaries}

We study multi-class classification and regression problems in this paper. In multi-class classification 
problems with $K \ge 2$ classes, we let $\cX$ be the input space and $\cY = [K] \triangleq \{ 1,2, \cdots, K\}$ 
be the output space, and the \textit{standard zero-one loss} $\zeroone(\yhat; y) = \ind(\yhat \neq y)$ is considered.

The regression problem is similar to the classification problem, except that the label becomes continuous, $\cY = [-1,1]$, and 
the goal is to learn a real-valued function $f : \cX \rightarrow \cY$ that approximates well labels of future instances. 
We consider the \textit{absolute loss} $\ell^{abs}(\hat{y}; y) = |\hat{y}-y|$ in this setting.
Results under the absolute loss can be generalized to any other Lipschitz losses with modified rates. 

\subsection{PAC learning}
Let $\cX$ be an input space, $\cY$ be an output space, and $\cD$ be an unknown distribution 
over $\cX \times \cY$. A \textit{hypothesis} is a function mapping from $\cX$ to $\cY$. 
The \textit{population loss} of a hypothesis $h : \cX \rightarrow \cY$ 
with respect to a loss function $\ell$ is defined by $\text{loss}_{\cD}(h) = \E_{(x,y) \sim \cD}\big[ \ell \big(h(x); y\big) \big]$.
We also define the \textit{empirical loss} of a hypothesis $h$ with respect to a 
loss function $\ell$ and a sample $S = \big((x_i, y_i)\big)_{1:n}$ as 
$\text{loss}_{S}(h) = \frac{1}{n} \sum_{i=1}^n \ell \big(h(x_i);y_i\big)$.
The distribution $\cD$ is said to be \textit{realizable} with respect to $\cH$ if there 
exists $h^{\star} \in \cH$ such that $\text{loss}_{\cD}(h^{\star})=0$.

\begin{definition}[PAC learning]
\label{def:PAC}
A hypothesis class $\cH$ is PAC learnable with sample complexity $m(\alpha, \beta)$ if 
there exists an algorithm $\cA$ such that for any $\cH$-realizable distribution $\cD$ over $\cX \times \cY$, 
an accuracy and confidence parameters $\alpha, \beta \in (0,1)$, if $\cA$ is given input 
samples $S =\big((x_i, y_i)\big)_{1:m} \sim \cD^m$ such that $m \ge m(\alpha, \beta)$, then 
it outputs a hypothesis $h : \cX \rightarrow \cY$ satisfying $\text{loss}_{\cD}(h) \le \alpha$ 
with probability at least $1-\beta$. A learner which always returns hypotheses inside the
class $\cH$ is called a proper learner, otherwise is called an improper learner.
\end{definition}

\subsection{Differential privacy}
\textit{Differential privacy} (DP) \citep{dwork2014algorithmic}, a standard notion of statistical data privacy, was introduced to study data analysis mechanism that do not reveal too much information 
on any single sample in a dataset.

\begin{definition}[Differential privacy \citep{dwork2014algorithmic}]
\label{def:DP}
Data samples $S, S^{\prime} \in (\cX \times \cY)^n$ are called neighboring 
if they differ by exactly one example. A randomized algorithm $\cA : (\cX \times \cY)^n \rightarrow \cY^{\cX}$ 
is $(\epsilon, \delta)$-differentially private if for all neighboring data samples $S, S^{\prime} \in (\cX \times \cY)^n$, 
and for all measurable sets $T$ of outputs,
\begin{align*}
\prob\big(\cA(S) \in T \big) \le e^{\epsilon} \cdot \prob\big(\cA(S^{\prime}) \in T \big) + \delta.
\end{align*}
The probability is taken over the randomness of $\cA$. When $\delta =0$ we say that $\cA$ preserves 
pure differential privacy, otherwise (when $\delta>0$) we say that $\cA$ preserves approximate 
differential privacy.
\end{definition}
Combining the requirements of PAC and DP learnability yields the definition of private PAC learner.
\begin{definition}[Private PAC learning \citep{kasiviswanathan2011can}]
\label{def:PPAC}
A hypothesis class $\cH$ is $(\epsilon,\delta)$-differentially private PAC learnable with sample complexity $m(\alpha, \beta)$ if 
it is PAC learnable with sample complexity $m(\alpha, \beta)$ by an algorithm $\cA$ 
which is $(\epsilon, \delta)$-differentially private.
\end{definition}

\subsection{Online learning}

The online learning problem can be viewed as a repeated game between a learner and an adversary. 
Let $T$ be a time horizon and $\cH \subset \cY^{\cX}$ be a class of predictors over a domain $\cX$. 
At time $t$, the adversary chooses a pair $(x_t, y_t) \in \cX \times \cY$, and
the learner observes the instance $x_t$, predicts a label $\hat{y}_t \in \cY$, and finally
observes the loss $\ell \big(\hat{y}_t ; y_t \big)$. 
This work considers the \textit{full-information setting} where the learner receives 
the true label information $y_t$. 
The goal is to minimize the \textit{regret}, 
namely the cumulative loss that the learner actually observed compared to the best prediction in hindsight:
\begin{align*}
\sum_{t=1}^T \ell \big(\hat{y}_t ; y_t \big) - \min_{h^{\star} \in \cH} \sum_{t=1}^T \ell \big(h^{\star}(x_t) ; y_{t}\big).
\end{align*}
A class $\cH$ is \textit{online learnable} if for every $T$, 
there is an algorithm that achieves sub-linear regret $o(T)$ against any sequence of $T$ instances. 

The \textit{Littlestone dimension} is a combinatorial parameter that exactly characterizes online learnability for binary 
hypothesis classes \citep{ben2009agnostic, littlestone1988learning}.
\citet{daniely2015multiclass} further extended this to the multi-class setting.
We need the notion of mistake trees to define this complexity measure. 
A \textit{mistake tree} is a binary tree whose internal nodes are labeled by elements of $\cX$.
Given a node $x$, its descending edges are labeled by distinct $k, \kprime \in \cY$. 
Then any root-to-leaf path can be expressed as a sequence of instances 
$\big((x_i,y_i)\big)_{1:d}$, where $x_i$ represents the $i$-th internal node in the path, and
$y_i$ is the label of its descending edge in the path.
We say that a tree $T$ is \textit{shattered} by $\cH$ if for any root-to-leaf path $\big((x_i,y_i)\big)_{1:d}$ of $T$,
there is $h \in \cH$ such that $h(x_i) =y_i$ for all $i \le d$.
The Littlestone dimension of multi-class hypothesis class $\cH$, 
$\Ldim (\cH)$, is the maximal depth of any $\cH$-shattered mistake tree.
Just like binary classification, a set of MC hypotheses $\cH$ is online learnable if and only if $\Ldim(\cH)$ is finite.
%

The (sequential) \textit{fat-shattering dimension} is the scale-sensitive complexity measure
for real-valued function classes \citep{rakhlin2015online}. 
A mistake tree for real-valued function 
class $\cF$ is a binary tree whose internal nodes are labeled by $(x, s) \in \cX \times \cY$,
where $s$ is called a \textit{witness to shattering}.  
Any root-to-leaf path in a mistake tree can be expressed as a sequence of tuples 
$\big((x_i, \epsilon_i)\big)_{1:d}$, where $x_i$ is the label of the $i$-th internal node in the path, 
and $\epsilon_i = +1$ if the $(i+1)$-th node is the right child of the $i$-th node, and otherwise $\epsilon_i= -1$ 
(for the leaf node, $\epsilon_{d}$ can take either value). 
A tree $T$ is $\gamma$-shattered by $\cF$ if for any root-to-leaf path $\big((x_i, \epsilon_i)\big)_{1:d}$ of $T$,
there exists $f \in \cF$ such that $\epsilon_i \left(f(x_i) - s_i\right) \ge \gamma/2 $ for all $i \le d$. 
The fat-shattering dimension at scale $\gamma$, denoted by $\fat_{\gamma}(\cF)$, 
 is the largest $d$ such that $\cF$ $\gamma$-shatters a mistake tree of depth $d$. 
For any function class $\cF \subset [-1,1]^{\cX}$, $\cF$ is online learnable in the supervised setting 
under the absolute loss if and only if $\fat_{\gamma}(\cF)$ is finite for any $\gamma >0$ \citep{rakhlin2015online}. 

The (sequential) \textit{Pollard pseudo-dimension} is a scale-free fat-shattering 
dimension for real-valued function classes. For every $f \in \cF$, 
we define a binary function $B_{f}:\cX \times \cY \rightarrow \{-1, +1\}$ by $B_f (x, s) = \textup{sign}\left(f(x) - s\right)$ and let $\cF^{+} = \{ B_f ~|~  f \in \cF \} $.
Then we define the Pollard pseudo-dimension by $\Pdim(\cF) = \Ldim(\cF^{+})$.
It is easy to check that $\fat_{\gamma}(\cF) \le \Pdim(\cF)$ for all $\gamma$. 
That being said, finite Pollard pseudo-dimension is a sufficient condition for online learnability but not a necessary condition (e.g., bounded Lipschitz functions on [0,1] separate the two notions).

\subsection{Additional notation}
\label{sec:additional.notation}
We define a few functions in a recursive manner. The \textit{tower function} $\twr_{t}$ and the \textit{iterated logarithm} $\log^{(m)}$ are defined respectively as
\begin{align*}
\twr_{t}(x) = 
\begin{cases}
	x &\text{ if } t = 0, \\
	2^{\twr_{t-1}(x)} &\text{ if } t > 0,
\end{cases}
\quad 
\log^{(m)}x=
\begin{cases}
	\log x &\text{ if } m = 1, \\
	\log^{(m-1)}\log x &\text{ if } m > 1.
\end{cases}
\end{align*}
Lastly, we use $\logstar x$ to denote the minimal number of recursions for the iterated logarithm to return the value less than or equal to one:
\begin{align*}
\logstar x=
\begin{cases}
	0 &\text{ if } x \le 1, \\
	1 + \logstar\log x &\text{ if } x > 1.
\end{cases}
\end{align*}

%% file: link.tex

\section{A link between multi-class and regression problems}
\label{sec:link}

As a tool to analyze regression problems, we discretize the continuous space $\cY$ into intervals and consider the problem as a multi-class problem. 
Specifically, given a function $f \in [-1, 1]^{\cX}$ and a scalar $\gamma$, we split the interval $[-1, 1]$ into $\lceil\frac{2}{\gamma}\rceil$ intervals of length $\gamma$ and define $[f]_{\gamma}(x)$ to be the index of interval that $f(x)$ belongs to. 
We can also define $[\cF]_{\gamma} = \{[f]_{\gamma}~|~f \in \cF\}$. 
In this way, if the multi-class problem associated with $[\cF]_{\gamma}$ is learnable, 
we can infer that the original regression problem is learnable up to accuracy $O(\gamma)$. 
Quite interestingly, however, the fact that $\cF$ is (regression) learnable does not imply that $[\cF]_{\gamma}$ is (multi-class) learnable. 
For example, it is well known that a class $\cF$ of bounded Lipschitz functions on [0,1] is learnable, but $[\cF]_{1}$ includes all binary functions on $[0,1]$, which is not online learnable. 

In order to tackle this issue, we propose a generalized zero-one loss in multi-class problems.
In particular, we define a \textit{zero-one loss with tolerance $\tau$},
\[
\zeroone_{\tau}(\yhat;y) = \ind(|y-\yhat| > \tau).
\]
Note that the classical zero-one loss is simply $\zeroone_{0}$. 
This generalized loss allows the learner to predict labels that are not equal to the true label but close to it. 
This property is well-suited in our setting since as far as $|y-\yhat|$ is small, 
the absolute loss in the regression problem remains small. 

We also extend the Littlestone dimension with tolerance $\tau$. 
Fix a tolerance level $\tau$.
When we construct a mistake tree $T$, we add another constraint that 
each node's descending edges are labeled by two labels $k, \kprime \in [K]$ such that $\zeroone_{\tau}(k ; \kprime) = 1$. 
Let $\Ldim_{\tau}(\cH)$ be the maximal height of such binary shattered trees. 
(Again, $\Ldim_{0}(\cH)$ becomes the standard $\Ldim(\cH)$.)

\begin{algorithm}[t]
	\begin{algorithmic}[1]
	    \STATE \textbf{Initialize:} $V_{0} = \cH$
	    \FOR{$t = 1, \cdots, T$}
	    \STATE Receive $x_{t}$
	    \STATE For $k \in [K]$, let $V_{t}^{(k)} = \{h \in V_{t-1} 
	    ~|~
	    h(x_{t}) = k\}$
	    \STATE Predict $\yhat_{t} = \argmax_{k} \Ldim_{\tau}(V_{t}^{(k)})$
	    \STATE Receive true label $y_{t}$ and update $V_{t} = V_{t}^{(y_{t})}$
	    \ENDFOR
	\end{algorithmic}
	\caption{Standard optimal algorithm with tolerance $\tau$ (\SOAt)}
	\label{alg:SOAt}
\end{algorithm}

We record several useful observations.
The proofs can be found in Appendix \ref{sec:link.app}.
\begin{lemma}
\label{lem:Ldim.t}
Let $\cH \subset [K]^{\cX}$ be a class of multi-class hypotheses.
\newcounter{item:tauLdim1}
\newcounter{item:tauLdim2}
\newcounter{item:tauLdim3}
\begin{enumerate}[leftmargin = *]
\item $\Ldim_{\tau}(\cH)$ is decreasing in $\tau$.
\setcounter{item:tauLdim1}{\value{enumi}}
\item \SOAt~(Algorithm \ref{alg:SOAt}) makes at most $\Ldim_{\tau}(\cH)$ mistakes with respect to $\zeroone_{\tau}$.
\setcounter{item:tauLdim2}{\value{enumi}}
\item For any deterministic learning algorithm, an adversary can force $\Ldim_{2\tau}(\cH)$ mistakes with respect to $\zeroone_{\tau}$.
\setcounter{item:tauLdim3}{\value{enumi}}
\end{enumerate}
\end{lemma}

Equipped with the relaxed loss, the following proposition connects regression learnability to multi-class learnability with discretization. 
We emphasize that even though the regression learnability does not imply multi-class learnability with the standard zero-one loss, learnability under $\zeroone_{\tau}$ can be derived. 
In addition to that, it can be shown that finite $\Ldim_{\tau}([\cF]_{\gamma})$ implies finite $\fat_{\gamma}(\cF)$.

\begin{proposition}
\label{prop:Ldim.fat}
Let $\cF \subset [-1, 1]^{\cX}$ be a regression hypothesis class
and suppose $\fat_{\gamma}(\cF) = d$.
Then we have for any positive integer $n$,
\[
\Ldim_{n}([\cF]_{\gamma/2(n+1)}) \ge d \ge \Ldim_{n}([\cF]_{\gamma/n}).
\]
\end{proposition}

\begin{proof}
Since $\fat_{\gamma}(\cF) = d$, in the online learning setting
an adversary can force any deterministic learner to suffer at least $
\gamma/2$ absolute loss for $d$ rounds. 
If we think of this problem as a multi-class classification problem using the hypothesis class $[\cF]_{\gamma/2(n+1)}$, 
using the same strategy, the adversary can force any deterministic learner to make mistakes with respect to $\zeroone_{n}$ for $d$ rounds. 
Note that the adversary reveals less information to the learner in the discretized multi-class problem.
Then Lemma~\ref{lem:Ldim.t} implies 
$\Ldim_{n}([\cF]_{\gamma/2(n+1)}) \ge d$.

On the other hand, suppose $\Ldim_{n}([\cF]_{\gamma/n}) > d$
and let $T$ be the binary shattered tree with tolerance $n$. 
For each node, we can set the witness point to be the middle point between the two labels of descending edges, 
and the resulting tree is $\gamma$-shattered by $\cF$. 
This contradicts the fact that $\fat_{\gamma}(\cF) = d$, and hence we obtain 
$d \ge \Ldim_{n}([\cF]_{\gamma/n})$. 
\end{proof}

There exist a few works that used regression models in multi-class classification \citep{rakesh2017ensemble, yang2005multi}.
To the best of our knowledge, however, our work is the first one that studies regression learnability by transforming the 
problem into a discretized classification problem along with a novel bridge, \textit{Littlestone dimension with tolerance}.

%% file: PL_OL.tex

\section{Private learnability implies online learnability}
\label{sec:PL.OL}

In this section, we show that if a class of functions is privately learnable, then it is online learnable. 
To do so, we prove a lower bound of the sample complexity of privately learning algorithms using 
either $\Ldim(\cH)$ for the multi-class hypotheses or $\fat_{\gamma}(\cF)$ for the regression hypotheses. 
\citet{alon2019private} proved this in the binary classification setting first by showing that any large Ldim class contains sufficiently many threshold functions and then providing a lower bound of the sample complexity to privately learn threshold functions. 
We adopt their arguments, but one of the first non-trivial tasks is to define analogues of threshold functions in multi-class or regression problems. Note that, a priori, it is not clear what the right analogy is.
Let us first introduce threshold functions in the binary case. 
We say a binary hypothesis class $\cH$ has $n$ thresholds if there exist $\{x_{i}\}_{1:n} \subset \cX$ and $\{h_{i}\}_{1:n} \subset \cH$ such that $h_{i}(x_{j}) = 1$ if $i \le j$ and $h_{i}(x_{j}) = 0$ if $i > j$. 
We extend this as below.

\begin{definition}[Threshold functions in multi-class problems]
\label{def:threshold.MC}
Let $\cH \subset [K]^{\cX}$ be a hypothesis class. 
We say $\cH$ contains $n$ thresholds with a gap $\tau$ if there exist $k, \kprime \in [K]$, $\{x_{i}\}_{1:n} \subset \cX$, and $\{h_{i}\}_{1:n} \subset \cH$ such that $|k -\kprime| > \tau$ and
$h_{i}(x_{j}) = k$ if $i \le j$ and $h_{i}(x_{j}) = \kprime$ if $i > j$. 
\end{definition}

\begin{definition}[Threshold functions in regression problems]
\label{def:threshold.reg}
Let $\cF \subset [-1, 1]^{\cX}$ be a hypothesis class. 
We say $\cF$ contains $n$ thresholds with a margin $\gamma$ if there exist $\{x_{i}\}_{1:n} \subset \cX$, $\{f_{i}\}_{1:n} \subset \cF$, and $u, \uprime \in [-1, 1]$ such that $|u - \uprime| \ge \gamma$ and
$|f_{i}(x_{j})-u| \le \frac{\gamma}{20}$ if $i \le j$ and $|f_{i}(x_{j})-\uprime| \le \frac{\gamma}{20}$ if $i > j$.
\end{definition}

In Definition \ref{def:threshold.reg}, we allow the functions to oscillate with a margin $\frac{\gamma}{20}$ which is arbitrary. 
Any small margin compared to $|u-\uprime|$ would work, 
but this number is chosen to facilitate later arguments. 

Next we show that complex hypothesis classes contain a sufficiently large set of threshold functions. 
The following theorem extends the results by \citet[Theorem 3]{alon2019private}.
A complete proof can be found in Appendix \ref{sec:PL.OL.app}.
\begin{theorem}[Existence of a large set of thresholds]
\label{thm:alon.thm3}
Let $\cH \subset [K]^{\cX}$ and $\cF \subset [-1, 1]^{\cX}$ be multi-class and regression hypothesis classes, respectively. 
\begin{enumerate}[leftmargin = *]
\item If $\Ldim_{2\tau}(\cH) \ge d$, 
then $\cH$ contains $\lfloor \frac{\log_{K} d}{K^{2}} \rfloor$ thresholds with a gap $\tau$.
\item If $\fat_{\gamma}(\cF) \ge d$, 
then $\cF$ contains $\lfloor \frac{\gamma^{2}}{10^{4}}\log_{100/\gamma} d \rfloor$ thresholds with a margin $\frac{\gamma}{5}$.
\end{enumerate}
\end{theorem}

\begin{proof}[Proof sketch]
We begin with the multi-class setting. 
Suppose $d = K^{K^{2}t}$. 
It suffices to show $\cH$ contains $t$ thresholds. 
Let $T$ be a shattered binary tree of height $d$ and tolerance $2\tau$. 
Letting $\cH_{0} = \cH$ and $T_{0} = T$, 
we iteratively apply \CNC ~(Algorithm \ref{alg:CNC}). 
Namely, we write 
\begin{equation}
\label{eq:thm3.MC1}
k_{n}, \kprime_{n}, h_{n}, x_{n}, \cH_{n}, T_{n} = \CNC(\cH_{n-1}, T_{n-1}, 2\tau).
\end{equation}

\newcounter{line:CNC1}
\newcounter{line:CNC2}
\newcounter{line:CNC3}
\begin{algorithm}[t]
	\begin{algorithmic}[1]
	    \STATE \textbf{Input:} multi-class hypothesis class $\cH \subset [K]^{\cX}$, shattered binary tree $T$, tolerance $\tau$
	    \STATE Choose an arbitrary hypothesis $h_{0} \in \cH$
	    \STATE Color each vertex $x$ of $T$ by $h_{0}(x) \in [K]$
	    \STATE Find a color $k$ such that the sub-tree $\Tprime \subset T$ of color $k$ has the largest height
	    \setcounter{line:CNC1}{\value{ALC@line}}
	    \STATE Let $x_{0}$ be the root node of $\Tprime$ 
	    \STATE Let $x_{1}$ be a child of $x_{0}$ such that the edge $(x_{0}, x_{1})$ is labeled as $\kprime$ with $|k-\kprime| > \frac{\tau}{2}$
	    \setcounter{line:CNC3}{\value{ALC@line}}
	    \STATE Let $\Tprimeprime$ be a sub-tree of $\Tprime$ rooted at $x_{1}$
	    \STATE Let $\cHprime = \{h \in \cH ~|~ h(x_{0}) = \kprime\}$
	    \setcounter{line:CNC2}{\value{ALC@line}}
	    \STATE \textbf{Output:} $k, k', h_{0}, x_{0}, \cHprime, \Tprimeprime$
	\end{algorithmic}
	\caption{\CNC}
	\label{alg:CNC}
\end{algorithm}

Observe that for all $n$, 
we can infer $h_{n}(x_{n}) = h_{n}(x) = k_{n}$ for all internal vertices $x$ of $T_{n}$ ($\because$ line \arabic{line:CNC1} of Algorithm \ref{alg:CNC})
and 
$h(x_{n}) = \kprime_{n}$ for all $h \in \cH_{n}$ ($\because$ line \arabic{line:CNC2} of Algorithm \ref{alg:CNC}).

Additionally, it can be shown that the height of $T_{n}$ is no less than $\frac{1}{K}$ times the height of $T_{n-1}$ (see Lemma \ref{lemma:color.subtree} in Appendix \ref{sec:PL.OL.app}).
This means that the iterative step \eqref{eq:thm3.MC1} can be repeated $K^{2}t$ times since $d = K^{K^{2}t}$. 
Then there exist $k, \kprime$ and indices $\{n_{i}\}_{i=1}^{t}$ such that 
$k_{n_{i}} = k$ and $\kprime_{n_{i}} = \kprime$ for all $i$.

It is not hard to check that the functions $\{h_{n_{i}}\}_{1:t}$ and the arguments $\{x_{n_{i}}\}_{1:t}$ form thresholds with labels $k, \kprime$.
Since $|k-\kprime| > \tau$ ($\because$ line \arabic{line:CNC3} of Algorithm \ref{alg:CNC}), this completes the proof.

The result in the regression setting can also be shown in a similar manner using Proposition \ref{prop:Ldim.fat}.
\end{proof}

\citet[Theorem 1]{alon2019private} proved a lower bound of the sample complexity in order to privately learn threshold functions. 
Then the multi-class result (with $\tau = 0$) of Theorem \ref{thm:alon.thm3} immediately implies that if $\cH$ is privately learnable, then it is online learnable. 
For the regression case, we need to slightly modify the argument to deal with the margin condition in Definition \ref{def:threshold.reg}. 
The next theorem summarizes the result, and the proof appears in Appendix \ref{sec:PL.OL.app}.
\begin{theorem}[Lower bound of the sample complexity to privately learn thresholds]
\label{thm:thm1.reg}
Let $\cF = \{f_{i}\}_{1:n} \subset [-1, 1]^{\cX}$ be a set of threshold functions with a margin $\gamma$ on a domain $\{x_{i}\}_{1:n} \subset \cX$ along with  bounds $u, \uprime \in [-1, 1]$. 
Suppose $\cA$ is a $(\frac{\gamma}{200}, \frac{\gamma}{200})$-accurate learning algorithm for $\cF$ with sample complexity $m$. 
If $\cA$ is $(\epsilon, \delta)$-DP with $\epsilon = 0.1$ and $\delta = O(\frac{1}{m^{2}\log m})$, then it can be shown that 
$m \ge \Omega(\log^{*}n)$.
\end{theorem}

Combining Theorem \ref{thm:alon.thm3} and \ref{thm:thm1.reg}, we present our main result. 
\begin{corollary}[Private learnability implies online learnability]
Let $\cH \subset [K]^{\cX}$ and $\cF \subset [-1, 1]^{\cX}$ be multi-class and regression hypothesis classes, respectively. 
Let $\Ldim(\cH) = \fat_{\gamma}(\cF) = d$. 
Suppose there is a learning algorithm $\cA$ that is $(\frac{1}{16}, \frac{1}{16})$-accurate for $\cH$ ($(\frac{\gamma}{200}, \frac{\gamma}{200})$-accurate for $\cF$) with sample complexity $m$. 
If $\cA$ is $(\epsilon, \delta)$-DP with $\epsilon = 0.1$ and $\delta = O(\frac{1}{m^{2}\log m})$, then 
$m \ge \Omega(\log^{*}d)$.
\end{corollary}

%% file: OL_PL_multiclass.tex

\section{Online learnability implies private learnability}\label{sec:OLPL}

In this section, we show that online-learnable multi-class hypothesis classes can be learned 
in a DP manner. 
For regression hypothesis classes, we provide sufficient conditions for private learnability.

\subsection{Multi-class classification}\label{subsec:OLPL_multiclass}

\citet{bun2020equivalence} proved that every binary hypothesis class with a finite Ldim 
is privately learnable by introducing a new notion of algorithmic stability called \textit{global stability} as 
an intermediate property between online learnability and differentially-private learnability. 
Their arguments can be naturally extended to MC hypothesis classes,
which is summarized in the next theorem.

\begin{theorem}[Online MC learning implies private MC learning] \label{thm:OLPL}
Let $\cH \subset [K]^{\cX}$ be a MC hypothesis class with $\Ldim(\cH) = d$.
Let $\epsilon, \delta \in (0,1)$ 
be privacy parameters and let $\alpha, \beta \in (0,1/2)$ be accuracy parameters. For
$n = O_d \big( \frac{\log (1/\beta \delta)}{\alpha \epsilon} \big)$, 
there exists an $(\epsilon, \delta)$-DP learning algorithm such that for every realizable distribution $\cD$, 
given an input sample $S \sim \cD^n$, the output hypothesis $f = \mathcal{A}(S)$ satisfies 
$\text{loss}_{\cD}(f) \le \alpha$ with probability at least $1-\beta$.
\end{theorem}

While we consider the realizable setting in Theorem \ref{thm:OLPL}, a similar result also holds in the agnostic setting. The extension to the agnostic setting is discussed in Appendix \ref{app:agnostic} due to limited space.
%
%

As a key to the proof of Theorem \ref{thm:OLPL}, we introduce global stability (GS) as follows.

\begin{definition}[Global stability \citep{bun2020equivalence}]
Let $n \in \mathbb{N}$ be a sample size and $\eta >0$ be a global stability parameter. 
An algorithm $\cA$ is $(n,\eta)$-GS with respect to $\cD$ if there exists a hypothesis $h$ such that 
$\prob_{S\sim \cD^n} \big(\cA(S) = h\big)\ge \eta$.
\end{definition}

Theorem \ref{thm:OLPL} can be proved in two steps. We first show that every MC hypothesis class with 
a finite Ldim is learnable by a GS algorithm $\cA$ (Theorem \ref{thm:OLGS}). 
Then we prove that any GS algorithm can be extended to a DP learning algorithm 
with a finite sample complexity.

\begin{theorem}[Online MC learning implies GS learning]\label{thm:OLGS}
Let $\cH \subset [K]^{\cX}$ be a MC hypothesis class with $\Ldim(\cH) = d$.
Let $\alpha>0$, and 
$m = \left((4K)^{d+1} + 1\right) \times [\frac{d \log K}{\alpha }]$. Then there exists a randomized algorithm 
$G : (\cX \times [K])^m \rightarrow [K]^{\cX}$ such that for a realizable distribution $\cD$ and an input sample 
$S \sim \cD^m$, there exists a $h$ such that 
\[
\prob\big(G(S) = h\big) \ge\frac{K-1}{(d+1)K^{d+1}} \quad \text{and} \quad loss_{\cD}(h) \le \alpha.
\]
\end{theorem}
Next, we give a brief overview on how to construct a GS learner $G$ and a DP learner $M$ in order to 
prove Theorem \ref{thm:OLPL}. The complete proofs are deferred to 
Appendix \ref{app:OLPL}.

\subsubsection{Online multi-class learning implies globally-stable learning}
Let $\cH$ be a MC hypothesis class with $\Ldim(\cH) =d$ and $\cD$ be a realizable distribution over 
examples $\big(x, c(x)\big)$ where $c \in \cH$ is an unknown target hypothesis. 
Recall that $\cH$ is learnable 
by $\SOA_0$ (Algorithm \ref{alg:SOAt}) with at most $d$ mistakes on any realizable sequence. 
Prior to building a GS learner $G$, we construct a distribution $\cD_k$ by appending $k$ \textit{tournament 
examples} between random samples from $\cD$, which force $\SOA_0$ to make at least $k$ 
mistakes when run on $S$ drawn from $\cD_k$. 
Using the fact that $\SOA_{0}$ identifies the true labeling function after making $d$ mistakes, 
we can show that
there exists $k \le d$ and a hypothesis $f : \cX \rightarrow [K]$ such that
\begin{equation*} \label{eq:error1}
\prob_{S \sim \cD_{k}, T \sim \cD^n} \big(\text{\SOA}_0(S \circ T) = f\big)\ge K^{-d}.
\end{equation*}

A GS learner $G$ is built by firstly drawing $k \in \{0, 1, \cdots, d\}$ uniformly at random and then running the $\SOA_0$ on $S \circ T$ where $S \sim \cD_{k}, T \sim \cD^n$. The learner $G$ outputs a good hypothesis that enjoys small population loss with probability at least $\frac{K^{-d}}{d+1}$. We defer the detailed construction of $\cD_k$ and proofs to Appendix \ref{app:OLPL}.

\subsubsection{Globally-stable learning implies private multi-class learning}

Let $G$ be a $(\eta, m)$-GS algorithm with respect to a target distribution $\cD$. We run $G$ on $k$ independent 
samples of size $m$ to non-privately produce a long list $H := (h_{i})_{1:k}$. The \textit{Stable Histogram} algorithm 
is a primary tool that
allows us to publish a short list of frequent hypotheses in a DP manner.
The fact that $G$ is GS ensures that some good hypotheses appear frequently in $H$. 
Then Lemma \ref{lemma:stable} implies that these good hypotheses remain in the short list with high probability. 
Once we obtain a short list, a generic DP learning algorithm
\citep{kasiviswanathan2011can} is applied to privately select an accurate hypothesis. 

\begin{lemma}[Stable Histogram \citep{dwork2006calibrating, korolova2009releasing}]\label{lemma:stable}
Let $X$ be any data domain. For $n \ge O(\frac{\log (1/\eta\beta\delta)}{\eta \epsilon})$, 
there exists an $(\epsilon, \delta)$-DP algorithm \textsc{Hist} which with probability 
at least $1-\beta$, on input $S = (x_i)_{1:n}$ outputs a list $L \subset X$ and 
a sequence of estimates $a \in [0,1]^{|L|}$ such that (i) every $x$ with $\textup{Freq}_{S}(x) \ge \eta$ appears in $L$, and
(ii) for every $x \in L$, the estimate $a_x$ satisfies $|a_x - \textup{Freq}_{S}(x)| \le \eta$ where $\textup{Freq}_{S}(x) := \big|\{i \in [n] ~|~ x_i = x \}\big|/n$.
\end{lemma}



%% file: OL_PL_regression.tex

\subsection{Regression}\label{sec:OLPLreg}
In classification, \textit{Global Stability} was an essential intermediate property between online and private learnability. 
A natural approach to obtaining a DP algorithm from an online-learnable real-valued function class $\cF$ is 
to transform the problem into a multi-class problem with $[\cF]_{\gamma}$ for some $\gamma$
and then construct a GS learner using the previous techniques. 
If $[\cF]_{\gamma}$ is privately-learnable, then we can infer that the original regression problem 
is also private-learnable up to an accuracy $O(\gamma)$. 

Unfortunately, however, finite $\fat_{\gamma}(\cF)$ only implies finite $\Ldim_{1}([\cF]_{\gamma})$, and $\Ldim([\cF]_{\gamma})$ can still be infinite (see Proposition \ref{prop:Ldim.fat}).
This forces us to run $\SOA_{1}$ instead of $\SOA_{0}$, and as a consequence, 
after making $\Ldim_{1}([\cF]_{\gamma})$ mistakes, the algorithm can identify the true function up to some tolerance. 
Therefore we only get the relaxed version of GS property as follows; there exist $k \le d$ and a hypothesis $f : \cX \rightarrow [K]$ such that
\begin{equation*}\label{eq:error2}
\prob_{S \sim \cD_{k}, T \sim \cD^n} \big(\SOA_1(S \circ T) \approx_1 f\big) \ge \left(\gamma/2\right)^{d} 
\end{equation*} 
where $f \approx_1 g$ means $\sup_{x \in \cX} \big|f(x) - g(x) \big| \le 1$. 
If we proceed with this relaxed condition, it is no longer guaranteed the long list $H$ contains a good hypothesis with sufficiently high frequency. 
This hinders us from using Lemma \ref{lemma:stable}, and a private learner cannot be produced in this manner. 
The limitation of proving the equivalence in regression stems from existing proof techniques.
With another method, it is still possible to show that online-learnable real-valued function classes can be learned by a DP algorithm. 
Instead, we provide sufficient conditions for private learnability in regression problems.


\begin{theorem}[Sufficient conditions for private regression learnability]\label{thm:sufficient}
Let $\cF \subset \cY^{\cX}$ be a real-valued function class such that
$\fat_{\gamma}(\cF) < \infty$ for every $\gamma>0$. If one of the following conditions holds, then $\cF$ is privately learnable.
\begin{enumerate}[leftmargin = *]
\item Either $\cF$ or $\cX$ is finite.
\item The range of $\cF$ over $\cX$ is finite (i.e., $\big|\{ f(x) ~|~ f \in \cF, x \in \cX \}\big| < \infty$).
\item $\cF$ has a finite cover with respect to the sup-norm at every scale.
\item $\cF$ has a finite sequential Pollard Pseudo-dimension.
\end{enumerate}
\end{theorem}
We present the proof of Condition 4, and proofs of other conditions are deferred to Appendix  \ref{app:suff}.
\begin{proof}[Proof of Condition 4]
Assume for contradiction that there exists $\gamma$ such that $\Ldim([\cF]_{\gamma}) =\infty$. 
Then we can obtain a shattered tree $T$ of an arbitrary depth. 
Choose an arbitrary node $x$. 
Note that its descending edges are labeled by $k, \kprime \in [\lceil2/\gamma\rceil]$.
We can always find a witness to shattering $s$ between the intervals corresponding to $k$ and $\kprime$. 
With these witness values, the tree $T$ must be zero-shattered by $\cF$.
Since the depth of $T$ can be arbitrarily large, this contradicts to $\Pdim(\cF)$ being finite. 
From this, we can claim that $\Ldim([\cF]_{\gamma}) \le \Pdim(\cF)$ for any $\gamma$. 
Then using the ideas in Section \ref{subsec:OLPL_multiclass}, we can conclude that $[\cF]_{\gamma}$ is private-learnable for any $\gamma$.
Therefore the original class $\cF$ is also private-learnable. 
\end{proof}
We emphasize that Conditions 3 and 4 do not imply each other. 
For example, a class of point functions $\cF^{\text{point}} := \{\ind(\cdot=x) ~|~ x \in \cX\}$ does not have a finite sup-norm cover because 
any two distinct functions have the sup-norm difference one, but $\Pdim(\cF^{\text{point}})=1$. 
A class $\cF^{\text{Lip}}$ of bounded Lipschitz functions on $[0, 1]$ has an infinite sequential Pollard pseudo-dimension, but $\cF^{\text{Lip}}$ has a finite cover 
with respect to the sup-norm due to compactness of $[0, 1]$ along with the Lipschitz property.

%% file: discussion.tex
\section{Discussion}
We have pushed the study of the equivalence between online and private learnability beyond binary classification. 
We proved that private learnability implies online learnability in the MC and regression settings. 
We also showed the converse in the MC setting and provided sufficient conditions
for an online learnable class to also 
be privately learnable in regression problems.

We conclude with a few suggestions for future work.
First, we need to understand whether online learnability implies private learnability in the regression setting.
Second, like \cite{bun2020equivalence}, we create an {improper DP} learner for an online learnable class.  It would be interesting to see if we can construct \textit{proper} DP learners.
Third, \citet{gonen2019private} provide an efficient black-box reduction from {\em pure} DP learning to online 
learning. It is natural to explore whether such efficient reductions are possible for {\em approximate} DP algorithms for MC and regression problems.
Finally, there are huge gaps between the lower and upper bounds for sample complexities in both classification and regression settings. 
It would be desirable to show tighter bounds and reduce these gaps. 

%% file: correction.tex

\section{Correction}
\label{sec:correction}

\citet{alon2019private} and \citet{bun2020equivalence} showed that online learnability and private PAC learnability are equivalent in binary classification, and we have extended their work to multi-class classification and regression. Recently, \citet{bun2020equivalence} discovered a technical mistake and presented a fix that deteriorates the dependence on the Littlestone dimension from exponential to doubly exponential. Accordingly, we also revisit this and present corrected results in this section.
The detailed proofs in Appendix \ref{app:OLPL} are corrected as well.

We provide not only a corrected version of Theorem \ref{thm:OLGS} but also a brief overview on how to construct a GS learner $G$ as in Section \ref{subsec:OLPL_multiclass} 
\newtheorem*{new.thm:OLGS2}{Theorem \ref{thm:OLGS}}
\begin{new.thm:OLGS2}[corrected]
Let $\cH \subset [K]^{\cX}$ be a MC hypothesis class with $\Ldim(\cH) = d$.
Let $\alpha>0$, and 
$m = \big(K^{2^{d+2}+1} \cdot 4^{d+1}  + 1\big) \times [\frac{2^{d+2} \log K}{\alpha }]$. Then there exists a randomized algorithm 
$G : (\cX \times [K])^m \rightarrow [K]^{\cX}$ such that for a realizable distribution $\cD$ and an input sample 
$S \sim \cD^m$, there exists a $h$ such that 
\[
\prob\big(G(S) = h\big) \ge\frac{K-1}{(d+1)K^{2^{d+2}+1}} \quad \text{and} \quad loss_{\cD}(h) \le \alpha.
\]
\end{new.thm:OLGS2}

Again, let $\cH$ be a MC hypothesis class with $\Ldim(\cH) =d$ and $\cD$ be a realizable distribution over 
examples $\big(x, c(x)\big)$ where $c \in \cH$ is an unknown target hypothesis. 
Recall that $\cH$ is learnable 
by $\SOA_0$ (Algorithm \ref{alg:SOAt}) with at most $d$ mistakes on any realizable sequence. 
First, we construct a distribution $\cD_k$ by appending $k$ \textit{tournament 
examples} between random samples from $\cD$, which force $\SOA_0$ to make at least $k$ 
mistakes when run on $S$ drawn from $\cD_k$. 
Using the fact that $\SOA_{0}$ identifies the true labeling function after making $d$ mistakes, 
we can show that
there exists $k \le d$ and a hypothesis $f : \cX \rightarrow [K]$ such that
\[
\prob_{S \sim \cD_{k}, T \sim \cD^n} \big(\text{\SOA}_0(S \circ T) = f\big)\ge K^{-2^{d+2}}.
\]

A GS learner $G$ is built by firstly drawing $k \in \{0, 1, \cdots, d\}$ uniformly at random and then running the $\SOA_0$ on $S \circ T$ where $S \sim \cD_{k}, T \sim \cD^n$. The learner $G$ outputs a good hypothesis that enjoys small population loss with probability at least $\frac{K-1}{d+1}K^{-2^{d+2}-1}$. We present a fix for the detailed construction of $\cD_k$ and proofs in Appendix \ref{app:OLPL}.


%% file: link_app.tex

\section{Section \ref{sec:link} details}
We prove Lemma \ref{lem:Ldim.t}.
\label{sec:link.app}

\newtheorem*{new.lem:Ldim.t}{Lemma \ref{lem:Ldim.t}}
\begin{new.lem:Ldim.t}[restated]
Let $\cH \subset [K]^{\cX}$ be a class of multi-class hypotheses.
\begin{enumerate}[leftmargin = *]
\item $\Ldim_{\tau}(\cH)$ is decreasing in $\tau$.
\item \SOAt~(Algorithm \ref{alg:SOAt}) makes at most $\Ldim_{\tau}(\cH)$ mistakes with respect to $\zeroone_{\tau}$.
\item For any deterministic learning algorithm, an adversary can force $\Ldim_{2\tau}(\cH)$ mistakes with respect to $\zeroone_{\tau}$.
\end{enumerate}
\end{new.lem:Ldim.t}

\begin{proof}
Part \arabic{item:tauLdim1} follows by observing that if $T$ is a binary shattered tree with tolerance $\tau$, then so is it with tolerance $\tau^{\prime} < \tau$. 

For part \arabic{item:tauLdim2}, assume \SOAt ~makes a mistake at round $t$. 
We claim that $\Ldim_{\tau}(V_{t+1}) < \Ldim_{\tau}(V_{t})$. 
If $\Ldim_{\tau}$ does not decrease, we can infer that 
\[
	\Ldim_{\tau}(V_{t}^{(\yhat_{t})}) = \Ldim_{\tau}(V_{t}^{(y_{t})}) = \Ldim_{\tau}(V_{t}) =: d.
\]
Then we can find binary trees $T_{1}$ and $T_{2}$ of height $d$ that are shattered by $V_{t}^{(\yhat_{t})}$ and $V_{t}^{(y_{t})}$, respectively. 
By concatenating $T_{1}$ and $T_{2}$ with a root node $x_{t}$ and its edges labeled by $\yhat_{t}$ and $y_{t}$, 
we can obtain a binary tree $T$ of height $d+1$ that is shattered by $V_{t}$. 
This contradicts to $\Ldim_{\tau}(V_{t}) = d$ and proves our assertion. 

To prove part \arabic{item:tauLdim3}, let $T$ be a binary shattered tree of height $\Ldim_{2\tau}(\cH)$. 
For a given node $x$, suppose the adversary shows $x$ to the learner.
Since the descending edges have labels apart from each other by more than $2\tau$, 
the adversary can choose a label that incurs a mistake with respect to $\zeroone_{\tau}$.
Thus by following down the tree $T$ from the root node, the adversary can force $\Ldim_{2\tau}(\cH)$ mistakes. 
\end{proof}

%% file: PL_OL_app.tex

\section{Section \ref{sec:PL.OL} details}
\label{sec:PL.OL.app}

In this section, the proofs omitted in Section \ref{sec:PL.OL} are presented.

\subsection{Proof of Theorem \ref{thm:alon.thm3}}
We first define \textit{sub-trees}. 
Let $T$ be a binary tree. 
Any node of $T$ becomes its sub-tree of height $1$. 
For $h > 1$, 
choose a node $x$ and let $T_{1}$ and $T_{2}$ be the trees that are rooted at its two children. 
A sub-tree of height $h$ is obtained by aggregating a sub-tree of height $h-1$ of $T_{1}$ 
and a sub-tree of height $h-1$ of $T_{2}$ at the root node $x$. 
Note that if the original tree $T$ is shattered by some hypothesis class, then so is any sub-tree of it. 

Next we prove a helper lemma. 

\begin{lemma}
\label{lemma:color.subtree}
Suppose there are $n$ colors $C = \{c_{i}\}_{1:n}$ and $n$ positive integers $\{d_{i}\}_{1:n}$.
Let $T$ be a binary tree of height 
$-(n - 1) + \sum_{i=1}^{n} d_{i}$ whose vertices are colored by $C$.
Then there exists a color $c_{i}$ such that $T$ has a sub-tree of height $d_{i}$ in which all internal vertices are colored by $c_{i}$.
\end{lemma}
\begin{proof}
We will prove by induction on $\sum_{i=1}^{n} d_{i}$.  
If $d_{i}=1$ for all $i$, then the height of $T$ becomes $1$, and the statement holds trivially. 
Now suppose the lemma holds for any $d_{i}$'s whose summation is less than $N$ 
and let $T$ have the height $N - n + 1$. 
Without loss of generality, we may assume that the root node $x_{0}$ is colored by $c_{1}$. 
We consider two sub-trees $T_{1}, T_{2}$ of height $N-n$ whose root nodes are children of $x_{0}$. 
Let $e_{1} = d_{1} - 1$ and $e_{i} = d_{i}$ for $i > 1$. 
Since $\sum_{i=1}^{n}e_{i} = N - 1$, 
by the inductive assumption each $T_{j}$ has a sub-tree of height $e_{i_{j}}$ in which all internal vertices are colored by $c_{i_{j}}$. 
If $i_{j} \neq 1$ for some $j$, then we are done because $e_{i_{j}} = d_{i_{j}}$. 
If $i_{j} = 1$ for all $j = 1, 2$, then merging these two trees with the node $x_{0}$ forms a sub-tree of height $e_{1}+1 = d_{1}$ of color $c_{1}$. 
This completes the inductive argument. 
\end{proof}

Now we are ready to prove Theorem \ref{thm:alon.thm3}.

\newtheorem*{new.thm:alon.thm3}{Theorem \ref{thm:alon.thm3}}
\begin{new.thm:alon.thm3}[restated]
Let $\cH \subset [K]^{\cX}$ and $\cF \subset [-1, 1]^{\cX}$ be multi-class and regression hypothesis classes, respectively. 
\begin{enumerate}[leftmargin = *]
\item If $\Ldim_{2\tau}(\cH) \ge d$, 
then $\cH$ contains $\lfloor \frac{\log_{K} d}{K^{2}} \rfloor$ thresholds with a gap $\tau$.
\item If $\fat_{\gamma}(\cF) \ge d$, 
then $\cF$ contains $\lfloor \frac{\gamma^{2}}{10^{4}}\log_{100/\gamma} d \rfloor$ thresholds with a margin $\frac{\gamma}{5}$.
\end{enumerate}
\end{new.thm:alon.thm3}

\begin{proof}
We begin with the multi-class setting. 
Suppose $d = K^{K^{2}t}$. 
It suffices to show $\cH$ contains $t$ thresholds. 
Let $T$ be a shattered binary tree of height $d$ and tolerance $2\tau$. 
Letting $\cH_{0} = \cH$ and $T_{0} = T$, 
we iteratively apply \CNC ~(Algorithm \ref{alg:CNC}). 
Namely, we write 
\begin{equation}
\label{eq:thm3.MC1.app}
k_{n}, \kprime_{n}, h_{n}, x_{n}, \cH_{n}, T_{n} = \CNC(\cH_{n-1}, T_{n-1}, 2\tau).
\end{equation}
Observe that for all $n$, 
we can infer $h_{n}(x_{n}) = h_{n}(x) = k_{n}$ for all internal vertices $x$ of $T_{n}$ ($\because$ line \arabic{line:CNC1} of Algorithm \ref{alg:CNC})
and 
$h(x_{n}) = \kprime_{n}$ for all $h \in \cH_{n}$ ($\because$ line \arabic{line:CNC2} of Algorithm \ref{alg:CNC}).

Additionally, Lemma \ref{lemma:color.subtree} ensures that the height of $T_{n}$ is no less than $\frac{1}{K}$ times the height of $T_{n-1}$.
This means that the iterative step \eqref{eq:thm3.MC1.app} can be repeated $K^{2}t$ times since $d = K^{K^{2}t}$. 
Then there exist $k, \kprime$ and indices $\{n_{i}\}_{i=1}^{t}$ such that 
$k_{n_{i}} = k$ and $\kprime_{n_{i}} = \kprime$ for all $i$.

It is not hard to check that the functions $\{h_{n_{i}}\}_{1:t}$ and the arguments $\{x_{n_{i}}\}_{1:t}$ form thresholds with labels $k, \kprime$.
Since $|k-\kprime| > \tau$ ($\because$ line \arabic{line:CNC3} of Algorithm \ref{alg:CNC}), this completes the proof.

Now we move on to the regression setting.
Proposition \ref{prop:Ldim.fat} implies that $\Ldim_{20}([\cF]_{\gamma/50}) \ge \Ldim_{24}([\cF]_{\gamma/50}) \ge d$.
Then using the previous result in the multi-class setting, we can deduce that 
$[\cF]_{\gamma/50}$ contains 
$n := \lfloor \frac{\gamma^{2}}{10^{4}}\log_{100/\gamma} d \rfloor$
thresholds with a gap $10$.
This means that there exist 
$k, \kprime \in [\frac{100}{\gamma}]$, $\{x_{i}\}_{1:n} \subset \cX$, and $\{[f_{i}]_{\gamma/50}\}_{1:n} \subset \cH$ such that $|k -\kprime| \ge 10$ and
\begin{align*}
	[f_{i}]_{\gamma/50}(x_{j}) = 
	\begin{cases}
	k &\text{ if } i \le j\\
	\kprime &\text{ if } i > j
	\end{cases}.
\end{align*} 
Let $u, \uprime$ be the middles points of the intervals that correspond to the labels $k, \kprime$. 
Then it is easy to check that $|u-\uprime| \ge \gamma/5$ and 
\begin{align*}
	f_{i}(x_{j}) \in 
	\begin{cases}
	[u - \frac{\gamma}{100}, u + \frac{\gamma}{100}) &\text{ if } i \le j\\
	[\uprime - \frac{\gamma}{100}, \uprime + \frac{\gamma}{100}) &\text{ if } i > j
	\end{cases}.
\end{align*}
This proves the theorem. 
\end{proof}

\subsection{Proof of Theorem \ref{thm:thm1.reg}}
\newtheorem*{new.thm:thm1.reg}{Theorem \ref{thm:thm1.reg}}
\begin{new.thm:thm1.reg}[restated]
Let $\cF = \{f_{i}\}_{1:n} \subset [-1, 1]^{\cX}$ be a set of threshold functions with a margin $\gamma$ on a domain $\{x_{i}\}_{1:n} \subset \cX$ along with  bounds $u, \uprime \in [-1, 1]$. 
Suppose $\cA$ is a $(\frac{\gamma}{200}, \frac{\gamma}{200})$-accurate learning algorithm for $\cF$ with sample complexity $m$. 
If $\cA$ is $(\epsilon, \delta)$-DP with $\epsilon = 0.1$ and $\delta = O(\frac{1}{m^{2}\log m})$, then it can be shown that 
$m \ge \Omega(\log^{*}n)$.
\end{new.thm:thm1.reg}
\begin{proof}
The proof consists of two main lemmas. 
Lemma \ref{lemma:homogeneous} proves that there is a large homogeneous set (see Definition \ref{def:homogeneous}).
Then Lemma \ref{lemma:lower.bound} yields the lower bound of the sample complexity when there exists a large homogeneous set. 
In particular, from these two lemmas, we can deduce that
\[
\frac{\log^{(m)}n}{2^{O(m \log m)}}
\le 
2^{O(m^{2}\log^{(2)}m)}.
\]
This means that there exists a constant $c$ such that 
\[
\log^{(m)}n \le e^{c m^{2}\log m}.
\]
Observing that $\logstar \big(\log^{(m)}n \big) \ge \big(\logstar n \big)- m$
and 
$\logstar \big(2^{O(m^{2}\log^{(2)}m)} \big) = O(\logstar m)$, 
we can check the desired inequality $m \ge \Omega (\logstar n)$. 
\end{proof}

\subsubsection{Existence of a large homogenous set}
Suppose $\cA$ is a learning algorithm over a finite domain $D$. 
The hypothesis class consists of threshold functions over $D$ with bounds $u, \uprime$. 
According to Definition \ref{def:threshold.reg}, $u$ and $\uprime$ can be in an arbitrary order as long as $|u - \uprime| > \gamma$. 
But for simpler presentation, 
without loss of generality, 
we will assume $u > \uprime$.
Also, let $\ubar = \frac{u + \uprime}{2}$. 
We define the following quantity:
\[
	\cA_{S}(x) = \prob_{f \sim \cA(S)}\big(f(x) \ge \ubar\big).
\]

The definition of homogenous sets (Definition \ref{def:homogeneous}) and Lemma \ref{lemma:homogeneous} are adopted from \citet{alon2019private}. 
Assume that $\cX$ is linearly ordered. 
Given a training set 
$S = \big((x_{i}, y_{i})\big)_{1:m}$,
we say $S$ is \textit{increasing} if $x_{1} \le \cdots \le x_{m}$.
Additionally, we say $S$ is \textit{balanced} if $y_{i} = \uprime$ for all $i \le \frac{m}{2}$ and $y_{i} = u$ for all $i > \frac{m}{2}$. 
Given $x \in \cX$, we define $\ord_{S}(x) = \big|\{i ~|~ x_{i} \le x\}\big|$. 
Lastly, we use $S_{\cX}$ to denote $(x_{i})_{1:m}$. 

\begin{definition}[$m$-homogeneous set]
\label{def:homogeneous}
A set $\Dprime \subset D$ is $m$-homogeneous with respect to a learning algorithm $\cA$ 
if there are numbers $p_{i} \in [0, 1]$ for $0 \le i \le m$ 
such that for every increasing balanced sample 
$S \in (\Dprime \times \{u, \uprime\})^{m}$
and for every $x \in \Dprime \setminus S_{\cX}$ 
\[
	|\cA_{S}(x) - p_{i}| \le \frac{1}{100m},
\]
where $i = \ord_{S}(x)$.
\end{definition}

The following theorem is a well-known result in Ramsey theory. 
It was originally introduced by \citet{erdos1952combinatorial} and rephrased by \citet{alon2019private}.
\begin{theorem}[{\citet[Theorem 11]{alon2019private}}]
\label{thm:ramsey}
Let $s > t \ge 2$ and $q$ be integers, and let $N \ge \twr_{t}(3sq \log q).$
Then for every coloring of the subsets of size $t$ of a universe of size $N$ 
using $q$ colors, 
there is a homogeneous subset \footnote{A subset of the universe is homogeneous if all of its $t$-subsets have the same color.} of size $s$. 
\end{theorem}

The next lemma states that we can find a large homogeneous set.
\begin{lemma}[Existence of a large homogeneous set]
\label{lemma:homogeneous}
Let $\cA$ be a learning algorithm over a domain $D$ with $|D| = n$. 
Then there exists a set $\Dprime \subset D$ which is $m$-homogeneous with respect to $\cA$ such that 
\[
|\Dprime| \ge \frac{\log^{(m)}n}{2^{O(m \log m)}}.
\]
\end{lemma}
\begin{proof}
We first define a coloring on the $(m+1)$-subsets of $D$. 
Let $B = \{x_{1} < x_{2} < \cdots < x_{m+1}\}$ be an $(m+1)$-subset. 
For each $i \in [m+1]$, let $B^{(i)} = B \setminus \{x_{i}\}$. 
Then by labeling the first half of $B^{(i)}$ by $\uprime$ and the second half by $u$,
we get a balanced increasing training set $S^{(i)}$. 
Then we compute $p_{i}$ that is of the form $\frac{t}{100m}$ and closest to $\cA_{S^{(i)}}(x_{i})$
(in case of ties, choose the smaller one).
Then we color $B$ by the tuple $(p_{i})_{1:m+1}$.

This scheme includes $(100m+1)^{m+1}$ colors, 
and Theorem \ref{thm:ramsey} provides that there exists a set $\Dprime$ of size larger than
\[
\frac{\log^{(m)}n}{3(100m+1)^{m+1}(m+1)\log(100m+1)}
=
\frac{\log^{(m)}n}{2^{O(m\log m)}}
\]
such that all $(m+1)$-subsets of $\Dprime$ have the same color. 
It is easy to verify that this set is indeed $m$-homogeneous with respect to $\cA$ according to Definition \ref{def:homogeneous}.
\end{proof}

\subsubsection{Large homogeneous set implies the lower bound}

Recall that PAC learning is defined with respect to $\text{loss}_{\cD}$ (see Definition \ref{def:PAC}).
When $\text{loss}_{\cD}$ is replaced by $\text{loss}_{S}$, 
we say an algorithm $\cA$ \textit{empirically learns} a training set $S$. 
\citet[Lemma 5.9]{bun2015differentially} prove that if a hypothesis class is PAC learnable, 
then there exists an empirical learner as well. 

\begin{lemma}[Empirical learner]
Suppose $\cA$ is an $(\epsilon, \delta)$-DP PAC learner for a hypothesis class $\cH$ that is $(\alpha, \beta)$-accurate and has sample complexity $m$. 
Then there is an $(\epsilon, \delta)$-DP and $(\alpha, \beta)$-accurate empirical learner for $\cH$ with sample complexity $9m$.
\end{lemma}

The next is the main lemma. 

\begin{lemma}[Large homogeneous sets imply lower bounds on sample complexity]
\label{lemma:lower.bound}
Suppose a learning algorithm $\cA$ is $(\epsilon, \delta)$-DP with sample complexity $m$. 
Let $X = [N]$ be $m$-homogeneous with respect to $\cA$. 
If $\epsilon = 0.1$, $\delta \le \frac{1}{1000m^{2}\log m}$, and $\cA$ empirically learns the threshold functions with a margin $\gamma$ over $X$ with $(\frac{\gamma}{200}, \frac{\gamma}{200})$-accuracy, then 
\[
N \le 2^{O(m^{2}\log^{(2)}m)}.
\]
\end{lemma}
\begin{proof}
The proof is done by combining Lemma \ref{lemma:lower.bound.helper1} and Lemma \ref{lemma:lower.bound.helper2}, which come below. 
\end{proof}

This is the first helper lemma to prove Lemma \ref{lemma:lower.bound}. 
It adopts \citet[Lemma 12]{alon2019private}.
\begin{lemma}
\label{lemma:lower.bound.helper1}
Let $\cA, X, m, N$ as in Lemma \ref{lemma:lower.bound} and assume $N > 2m$. 
Then there exists a family $\cP = \{P_{i}\}_{1:N-m}$ of distributions over $\{-1, 1\}^{N-m}$ that satisfies the following two properties.
\begin{enumerate}
\item $P_{i}$ and $P_{j}$ are $(\epsilon, \delta)$-indistinguishable for all $i \ne j$.
\item There exists $r \in [0, 1]$ such that for all $i, j \in [N-m]$, 
\begin{align*}
	\prob_{v \sim P_{i}} (v_{j} = 1) 
	\begin{cases}
	\le r - \frac{1}{10m} &\text{ if } j < i\\
	\ge r + \frac{1}{10m} &\text{ if } j > i
	\end{cases}.
\end{align*}
\end{enumerate}
\end{lemma}
\begin{proof}
Let $(p_{i})_{0:m}$ be the probability list associated with $m$-homogeneous set $X = [N]$. 
We first prove that there exists $\istar$ such that $p_{\istar} - p_{\istar-1} \ge \frac{1}{4m}$. 
Fix an increasing balanced training set $S := \big((x_{i}, y_{i})\big)_{1:m} \in \big(X \times \{u, \uprime\}\big)^{m}$
such that $x_{i} - x_{i-1} \ge 2$ for all $i$, 
which is possible by the assumption $N > 2m$. 
By the definition of threshold functions with a margin $\gamma$, 
we can infer
\[
\min_{f} \text{loss}_{S}(f) \le \frac{\gamma}{20} = 0.05 \gamma,
\]
where the minimum is taken over the threshold functions with a margin $\gamma$. 

Furthermore, since $\cA$ is an $(\alpha = \frac{\gamma}{200}, \beta = \frac{\gamma}{200})$-accurate empirical learner, we can bound the expected loss of $\cA(S)$ as
\begin{equation}
\label{eq:lower.bound.helper1.1}
\E_{f \sim \cA(S)} \text{loss}_{S}(f)
\le \alpha + \beta + \min_{f} \text{loss}_{S}(f)
\le 0.06\gamma.
\end{equation}
Also, we can lower bound the expected empirical loss by using the quantity $\cA_{S}(x_{i})$ as follows (recall that we assumed $u > \uprime$)
\begin{equation}
\label{eq:lower.bound.helper1.2}
\E_{f \sim \cA(S)} \text{loss}_{S}(h)
\ge
\frac{1}{m}\cdot \frac{\gamma}{2}\left(
\sum_{i=1}^{m/2} \left[ \cA_{S}(x_{i}) \right]
+
\sum_{i=m/2 +1}^{m} \left[ 1 - \cA_{S}(x_{i}) \right]
\right).
\end{equation}
Combining \eqref{eq:lower.bound.helper1.1} and \eqref{eq:lower.bound.helper1.2}, 
we can show that there exists $j \le \frac{m}{2}$ such that 
$\cA_{S}(x_{j}) \le \frac{1}{4}$. 
Let $\Sprime = \left(S \setminus \{(x_{j}, y_{j})\} \right)\cup \{(x_{j}+1, y_{j})\}$.
Since $\cA$ is $(\epsilon=0.1, \delta \le \frac{1}{1000m^{2}\log m})$-DP, 
we have 
\[
p_{j-1} - \frac{1}{100m} 
\le\cA_{\Sprime}(x_{j})
\le \frac{1}{4}e^{\epsilon} + \delta 
\le 0.3,
\] 
which implies that $p_{j-1} \le 0.3 + \frac{1}{100m} \le \frac{1}{3}$. 
Similarly, we can find $k > \frac{m}{2}$ such that $p_{k+1} \ge \frac{2}{3}$.
Then we can find $\istar \in [j, k+1]$ such that $p_{\istar} - p_{\istar-1} \ge \frac{1}{4m}$,
which proves our assertion.

Now we construct $\cP = \{P_{i}\}_{1:N-m}$.
Given $i$, let 
$$
B^{(i)} = \{1, \cdots, \istar-1\} \cup \{\istar +i\} \cup \{\istar + N -m +1, \cdots, N\} 
\subset X.
$$
Observe that $B^{(i)}$ and $B^{(j)}$ only differ by one item at the position $\istar$. 
Then define $S^{(i)}$ to be the balanced increasing training set built upon $B^{(i)}$. 
Given a hypothesis $f$, 
we can compute a $N-m$ dimensional binary vector $v \in \{-1, 1\}^{N-m}$ such that 
$$
v_{j} = \ind \left(f(\istar-1 +j) \ge \ubar\right)
\text{, where }
\ubar =  \frac{u+\uprime}{2}.
$$
This mapping induces a distribution over $\{-1, 1\}^{N-m}$ from $\cA(S^{(i)})$, which we define to be $P_{i}$.

Due to DP property of $\cA$, $P_{i}$ and $P_{j}$ are $(\epsilon, \delta)$-indistinguishable. 
Furthermore, our construction of $\istar$ ensures the second property with $r = \frac{p_{i-1}+p_{i}}{2}$. 
This completes the proof. 
\end{proof}

The second helper lemma is shown by \citet[Lemma 13]{alon2019private}.
\begin{lemma}
\label{lemma:lower.bound.helper2}
Suppose the family $\cP$ as in Lemma \ref{lemma:lower.bound.helper1} exists. 
Then $N-m \le 2^{1000m^{2}\log^{(2)}m}$.
\end{lemma}

%% file: OL_PL_app.tex

\section{Section \ref{sec:OLPL} details} \label{app:OLPL}
We provide details omitted in Section \ref{sec:OLPL}\footnote{This section is corrected according to the fix in Section \ref{sec:correction}. The primary change is the bounds in Appendix \ref{app:OLGS} are changed from exponential in $K$ to doubly exponential in $K$.}.

\subsection{Proof of Theorem \ref{thm:OLGS}} \label{app:OLGS}

Let $\cH$ be a multi-class hypothesis class with $\Ldim(\cH) =d$ and $\cD$ be a realizable distribution over 
examples $(x, c(x))$ where $c \in \cH$ is an unknown target hypothesis. 
The globally-stable (GS) leaner $G$ for $\cH$ will make use of the Standard Optimal Algorithm ($\SOA_0$, Algorithm \ref{alg:SOAt}).

$\SOA_0$ can be simply extended to non-realizable sequences as follows.
\begin{definition}[Extending the $\SOA_0$ to non-realizable sequences]
Consider a run of $\SOA_0$ on examples $\big((x_i,y_i)\big)_{1:m}$, and let $h_t$ denote the predictor 
used by the $\SOA_0$ after observing the first $t$ examples. Then after observing $(x_{t+1}, y_{t+1})$, proceed as below.
\begin{itemize}[leftmargin = *]
\item If $\big((x_i,y_i)\big)_{1:t+1}$ is realizable by some $h \in \cH$, 
then apply the usual update rule of the $\SOA_0$ to obtain $h_{t+1}$.
\item Else, set $h_{t+1}$ as $h_{t+1}(x_{t+1}) = y_{t+1}$, and $h_{t+1}(x) = h_t(x)$ for every $x \neq x_{t+1}$.
That is to say, $h_{t+1}$ no longer belongs to $\cH$. 
\end{itemize}
\end{definition}
This update rule keeps updating the predictor $h_t$ to agree with the last example while 
observing the sequences which are not necessarily realized by a hypothesis in $\cH$.
Due to this extension, our resulting algorithm possibly becomes improper. 

The finite Littlestone class is online learnable 
by $\SOA_0$ (Algorithm \ref{alg:SOAt}) with at most $d$ mistakes on any realizable sequence. 
Prior to building a GS learner $G$, we define a distribution $\cD_k$ as in Algorithm \ref{alg:Dconstruction}.

\begin{algorithm}[ht]
    \caption{Distribution $\cD_k$}
    \begin{algorithmic}[1]\label{alg:Dconstruction}
        \STATE $\cD_0$ : output an empty set with probability 1
        \STATE Let $k \ge 1$. If there exists an $f$ satisfying 
        $\prob_{S \sim \cD_{k-1}, T \sim \cD^n} \big(\text{\SOA}_0(S \circ T) = f\big) \ge K^{-2^{d+2}}$, \\
        or if $\cD_{k-1}$ is undefined, then $\cD_{k}$ is undefined
        \STATE Else, $\cD_{k}$ is defined recursively as follows
        \STATE \quad (i) Randomly sample $S_0, S_1 \sim \cD_{k-1}$ and $T_0, T_1 \sim \cD^n$
        \STATE \quad (ii) Let $f_0 = \text{\SOA}_0(S_0 \circ T_0)$ and $f_1 = \text{\SOA}_0(S_1 \circ T_1)$
        \STATE \quad (iii) If $f_0 = f_1$, go back to step (i)
        \STATE \quad (iv) Else, pick $ x \in \{x ~|~ f_0(x) \neq f_1(x) \}$ and sample $y \sim [K]$ uniformly at random
        \STATE \quad (v) If $f_0(x) \neq y$, output $S_0 \circ T_0 \circ (x,y)$ and $S_1 \circ T_1 \circ (x,y)$ otherwise
    \end{algorithmic}
\end{algorithm}

Let $k$ be such that $\cD_k$ is well-defined and consider a sample $S$ drawn from $\cD_k$. 
The size of $\cD_k$ is $k \cdot (n+1)$, and they consist of $k\cdot n$ instances randomly drawn from $\cD$ and 
$k$ examples generated in Item 3(iv) of Algorithm \ref{alg:Dconstruction}.
We call these $k$ examples \textit{tournament examples}.
Due to the construction of $\cD_k$, 
$\SOA_{0}$ always errs in tournament rounds, which means that $\SOA_0$ makes at least $k$ 
mistakes when run on $S \circ T$ where $S \sim \cD_k, T\sim  \cD^n$.

A natural way to obtain a GS learning algorithm $G$ is to run the $\text{SOA}_0$ on this carefully chosen sample $S\circ T$. 
In fact, the output enjoys both global stability in multi-class learning and good generalization as follows.

\begin{lemma}[Global Stability]\label{lemma:gs}
There exist $k \le d$ and a hypothesis $f : \cX \rightarrow [K]$ such that
\[
\prob_{S \sim \cD_{k}, T \sim \cD^n} \big(\text{\SOA}_0(S \circ T) = f\big) \ge K^{-2^{d+2}}.
\]
\end{lemma}
\begin{proof}
Assume for contradiction that $\cD_d$ is well-defined and for every $f$, 
\[
\prob_{S \sim \cD_{k}, T \sim \cD^n} \big(\text{\SOA}_0(S \circ T) = f\big) < K^{-2^{d+2}}.
\]
We prove that this cannot be the case when $f=c$ is the target concept.
First, we show that with probability $K^{-2^{d+2}}$ over $S \sim \cD_d$ all $d$ tournament examples are consistent with $c$.
For $k \le d$ let $\rho_k$ be  the probability that all $k$ tournament examples over $S \sim \cD_k$ are consistent with $c$. We claim that $\rho_k$ satisfies the recursion $\rho_k \ge \frac{1}{K} (\rho_{k-1}^2 - 2 \cdot K^2 \cdot K^{-2^{d+2}})$. 
Let $E_k$ be the event that (i) in each of $S_0, S_1 \sim \cD_{k-1}$, all $k-1$ tournament examples are consistent with $c$, and (ii) $f_0 \neq f_1$. 
By our initial assumption, we have $f_0 = f_1$ with probability at most $K^{-2^{d+2}} < 2 \cdot K^2 \cdot K^{-2^{d+2}}$,
and it follows that $\prob(E_k) \ge \rho_{k-1}^2 - 2 \cdot K^2 \cdot K^{-2^{d+2}}$. Since $y \in [K]$ is chosen uniformly at random and independently of $S_0$ and $S_1$, we have that $c(x) = y$ with probability $1/K$ conditioned on $E_k$. Accordingly, we have the following recursive relation $\rho_0 =1 $ and 
$$
\rho_k \ge \frac{1}{K} \prob(E_k) \ge \frac{1}{K} (\rho_{k-1}^2 - 2\cdot K^2 \cdot K^{-2^{d+2}}).
$$
Then we can prove by induction that for $k \le d, \rho_k \ge 2\cdot K \cdot K^{-2^{k+1}}$: the base case is readily verified, and the induction step is as follows:
\begin{align*}
	\rho_k &\ge \frac{1}{K} (\rho_{k-1}^2 -  2\cdot K^2 \cdot K^{-2^{d+2}}) \\
	& \ge \frac{1}{K} \big( (2\cdot K \cdot K^{-2^{k}})^2- 2 \cdot K^2 \cdot K^{-2^{d+2}}\big)\\
	& =  4 \cdot K \cdot K^{-2^{k+1}} - 2 \cdot K \cdot K^{-2^{d+2}} \\
	& \ge 2 \cdot K \cdot K^{-2^{k+1}}.
\end{align*}
The last inequality holds since $k \le d$ and therefore $K^{-2^{d+2}} \le K^{-2^{k+1}}$.

Accordingly, with probability $K^{-2^{d+2}}$ over $S \sim \cD_d$, all $d$ tournament examples are consistent with the true labeling function $c$ and thus $S \circ T$ becomes consistent with $c$. Since the number of total mistakes of $\text{\SOA}_0$ should be no more than $d$, we can deduce that 
$\text{\SOA}_0(S \circ T) =c$. This implies that
\[
\prob_{S \sim \cD_{k}, T \sim \cD^n} \big(\text{\SOA}_0(S \circ T) = c\big) \ge K^{-2^{d+2}},
\]
which is a contradiction, and hence completes the proof.
\end{proof}

\begin{lemma}[Generalization]\label{lemma:generalization}
Let $k$ be such that $\cD_k$ is well-defined. Then for every $f$ such that
\[
\prob_{S \sim \cD_{k}, T \sim \cD^n} \big(\text{\SOA}_0(S \circ T) = f\big) \ge  K^{-2^{d+2}}
\]
satisfies $\text{loss}_{\cD}(f) \le \frac{2^{d+2} \log K}{n}$.
\end{lemma}

\begin{proof}
Let $f$ be such hypothesis and let $\alpha = \text{loss}_{\cD}(f)$. We argue that $K^{-2^{d+2}} \le (1-\alpha)^n$. Then the following result is derived, $\alpha \le  \frac{2^{d+2} \log K}{n}$ using the fact that $(1-\alpha)^n \le e^{-n\alpha}$. 

By the property of $\SOA_{0}$, $\SOA_0(S \circ T)$ is consistent with $T$.
Thus, if $\text{\SOA}_0(S \circ T)=f$, then it must be the case that $f$ is consistent with $T$. By assumption, $\text{\SOA}_0(S \circ T)=f$ holds with probability at least $K^{-2^{d+2}}$ and $f$ is consistent with $T$ with probability $(1-\alpha)^n$ where $n$ is the size of $T$.
This gives the desired inequality.
\end{proof}

One challenge associated with the distribution $\cD_k$ is computational limitation.
It may require an unbounded number of samples from the target distribution $\cD$, 
since during generation of tournament examples 
the number of samples drawn from $\cD$ depends on how many times
Item 3(i)-(iii) will be repeated. 
To handle this practical issue, we suggest a Monte-Carlo Variant of $\cD_k$, $\tilde{\cD}_k$, by setting an upper bound $N$ of random samples drawn from $\cD$ as an input parameter. 
Algorithm \ref{alg:MCconstruction} summarizes how we construct the distribution $\tilde{\cD}_k$.

\begin{algorithm}[h]
    \caption{Distribution $\tilde{\cD}_k$}
    \begin{algorithmic}[1]\label{alg:MCconstruction}
    	\STATE Let $n$ be the auxiliary sample size and $N$ be an upper bound on the number of samples from $\cD$
        \STATE $\tilde{\cD}_0$ : output an empty set with probability 1
        \STATE Let $k \ge 1$. $\tilde{\cD}_{k}$ is defined recursively by the following processes
        \STATE \quad ($\star$) Throughout the process, if more than $N$ examples are drawn from $\cD$, then output ``Fail''
        \STATE \quad (i) Randomly sample $S_0, S_1 \sim \tilde{\cD}_{k-1}$ and $T_0, T_1 \sim \cD^n$
        \STATE \quad (ii) Let $f_0 = \text{\SOA}_0(S_0 \circ T_0)$ and $f_1 = \text{\SOA}_0(S_1 \circ T_1)$
        \STATE \quad (iii) If $f_0 = f_1$, go back to step (i)
        \STATE \quad (iv) Else, pick $x \in \{x ~|~ f_0(x) \neq f_1(x) \}$ and sample $y \sim [K]$ uniformly at random
        \STATE \quad (v) If $f_0(x) \neq y$, output $S_0 \circ T_0 \circ (x,y)$ and $S_1 \circ T_1 \circ (x,y)$ otherwise
    \end{algorithmic}
\end{algorithm}

The next step is to specify the upper bound $N$. The following lemma characterizes the expected sample complexity of sampling from $\cD_k$.

\begin{lemma}[Expected sample complexity of sampling from $\cD_k$]\label{lemma:complex}
Let $k$ be such that $\cD_k$ is well-defined and $M_k$ be the number of samples from $\cD$ when generating 
$S \sim \cD_k$.
Then we have $\E M_k \le 4^{k+1}\cdot n$.
\end{lemma}
\begin{proof}
Initially, $\E M_0 = 0$ since $\cD_0$ outputs an empty set with probability 1. It suffices to show that for all $0<i<k$, $\E M_{i+1} \le 4 \E M_i + 4n$ to conclude the desired inequality by induction.

Let $R$ be the number of times Item 3(i) was executed during generation of $S \sim \cD_{i+1}$, 
and $R$ is distributed geometrically with a success probability $\theta$, where
\begin{align*}
\theta &= 1- \prob_{S_0, S_1, T_0, T_1} \big( \text{\SOA}_0(S_0 \circ T_0) =  \text{\SOA}_0(S_1 \circ T_1) \big) \\
& \ge  1 - \max_f  \Big(\prob_{S, T} \big( \text{\SOA}_0(S \circ T) =  f\big)\Big)  \\
&\ge 1 -   K^{-2^{d+2}}.
\end{align*}
The last inequality holds because $i < k$ and hence $\cD_i$ is well-defined, which implies that
$\prob_{S, T} \big( \text{\SOA}_0(S \circ T) =f\big) \le K^{-2^{d+2}}$ for all $f$.

Let $M_{i+1}$ be a random variable expressed as $M_{i+1} = \sum_{j=1}^{\infty} M_{i+1}^{(j)}$ where
$$
M_{i+1}^{(j)} = \begin{cases}
0, & \text{if}~ R < j  \\
\text{the number of examples from $\cD$ in the $j$-th execution of Item 3(i)}, & \text{if}~ R \ge j
\end{cases}.
$$
Thus, we have
\begin{align*}
\E M_{i+1} 
& = \sum_{j=1}^{\infty} \E M_{i+1}^{(j)} = \sum_{j=1}^{\infty} (1-\theta)^{j-1} \cdot (2 \E M_i + 2n) \\ 
 & = \frac{1}{\theta} \cdot  (2 \E M_i + 2n)  \le 4 \E M_i + 4n,
\end{align*}
where the last inequality holds since $\theta \ge 1- K^{-2^{d+2}} \ge 1/2$ 
since $K \ge 2$ and $d \ge 1$.
\end{proof}

Equipped with Lemma \ref{lemma:gs},\ref{lemma:generalization}, and \ref{lemma:complex}, we are ready to prove Theorem \ref{thm:OLGS}.

\newtheorem*{new.thm:OLGS}{Theorem \ref{thm:OLGS}}
\begin{new.thm:OLGS}[restated]
Let $\cH \subset [K]^{\cX}$ be a MC hypothesis class with $\Ldim(\cH) = d$.
Let $\alpha>0$, and 
$m = \big(K^{2^{d+2}+1} \cdot 4^{d+1}  + 1\big) \times [\frac{2^{d+2} \log K}{\alpha }]$. Then there exists a randomized algorithm 
$G : (\cX \times [K])^m \rightarrow [K]^{\cX}$ such that for a realizable distribution $\cD$ and an input sample 
$S \sim \cD^m$, there exists a $h$ such that 
\[
\prob\big(G(S) = h\big) \ge\frac{K-1}{(d+1)K^{2^{d+2}+1}} \quad \text{and} \quad loss_{\cD}(h) \le \alpha.
\]
\end{new.thm:OLGS}

\begin{proof}

The globally-stable algorithm $G$ is defined in Algorithm \ref{alg:Gconstruction}.

\begin{algorithm}[h]
    \caption{Algorithm $G$}
    \begin{algorithmic}[1]\label{alg:Gconstruction}
        \STATE \textbf{Input :} target distribution $\tilde{\cD}_k$, auxiliary sample size $n = [\frac{2^{d+2} \log K}{\alpha}]$, and the sample complexity upper bound $N = K^{2^{d+2}+1} \cdot 4^{d+1} \cdot n$
        \STATE Draw $k \in \{0, 1, \cdots, d \}$ uniformly at random
        \STATE \textbf{Output :} $h = \text{\SOA}_0(S \circ T)$, where $T \sim \cD^n, S \sim \tilde{\cD}_k$
    \end{algorithmic}
\end{algorithm}
The sample complexity of $G$ is $|S| + |T| \le N + n = \big(K^{2^{d+2}+1} \cdot 4^{d+1}  + 1\big) \times [\frac{2^{d+2} \log K}{\alpha}]$. By Lemma \ref{lemma:gs} and \ref{lemma:generalization}, there exists $k^{\star} \le d$ and $f^{\star}$ such that
\[
\prob_{S \sim \cD_{k^{\star}}, T \sim \cD^n} \big(\text{\SOA}(S \circ T) =f^{\star}\big) \ge K^{-2^{d+2}}, \quad \text{loss}_{\cD}(f^{\star}) \le \frac{2^{d+2} \log K}{n} \le \alpha.
\]
We claim that $G$ outputs $f^{\star}$ with probability at least $(K-1)K^{-2^{d+2}-1}$. Let $M_{k^{\star}}$ denote the number of random examples from $\cD$ during generation of $S \sim \cD_{k^{\star}}$. 
We obtain the following inequality from Lemma \ref{lemma:complex} and Markov's inequality,
\begin{align*}
    \prob \big(M_{k^{\star}} > K^{2^{d+2}+1} \cdot 4^{d+1}  \cdot n \big) 
    & \le \prob\big( M_{k^{\star}} > K^{2^{d+2}+1} \cdot 4^{k^{\star}+1}  \cdot n  \big)\\
   &  \le K^{-2^{d+2}-1}.
\end{align*}
Accordingly,
\begin{align*}
    \prob_{S \sim \tilde{\cD}_{k^{\star}}, T \sim \cD^n} &\big(\text{\SOA}_0(S \circ T) =f^{\star}\big)  \\
    & \ge \prob_{S \sim \cD_{k^{\star}}, T \sim \cD^n} \big(\text{\SOA}_0(S \circ T) =f^{\star} \; \text{and} \; M_{k^{\star}} \le K^{2^d+2} \cdot 4^{d+1}\cdot n\big)\\
    & \ge \prob_{S \sim \cD_{k^{\star}}, T \sim \cD^n} \big(\text{\SOA}_0(S \circ T) =f^{\star} \big) - \prob \big( M_{k^{\star}} > K^{2^d+2} \cdot 4^{d+1}\cdot n \big) \\
    &\ge K^{-2^{d+2}} - K^{-2^{d+2}-1} = (K-1)K^{-2^{d+2}-1} 
\end{align*}
Since $k = k^{\star}$ with probability $\frac{1}{d+1}$, $G$ outputs $f^{\star}$ with probability at least $ \frac{(K-1)K^{-2^{d+2}-1}}{d+1} $.
\end{proof}

\subsection{Globally-stable learning implies private multi-class learning} \label{app:GSPL}

In this section, we utilize the GS algorithm from the previous section to derive a DP learning algorithm with a finite sample complexity. Theorem \ref{thm:OLPL} establishes that online multi-class learnability implies private multi-class learnability, which can be proved by combining Theorem \ref{thm:OLGS} and Theorem \ref{thm:GSPL}.  

\begin{theorem}[Globally-stable learning implies private multi-class learning]\label{thm:GSPL}
Let $\cH \subset [K]^{\cX}$ be a multi-class hypothesis class. Let $G : (\cX \times [K])^m \rightarrow[K]^{\cX}$ be a 
randomized algorithm such that for a realizable distribution $\cD$ and $S \sim \cD^m$, there exists a hypothesis $h$ 
such that $\prob \big(G(S) = h\big) \ge \eta$  and $\text{loss}_{\cD}(h) \le \alpha/2$. Then for some 
$n = O(\frac{m \log (1/\eta\beta\delta)}{\eta \epsilon} + \frac{\log (1/\eta\beta)}{\alpha \epsilon})$, 
there exists an $(\epsilon, \delta)$-DP algorithm $M$ which for $n$ i.i.d. samples from $\cD$, 
outputs a hypothesis $\hat{h}$ such that $\text{loss}_{\cD}(\hat{h}) \le \alpha$ with probability at least $1-\beta$.
\end{theorem}

To construct a private learner $M$, we first introduce standard tools in the DP community such as \textit{Stable Histogram} 
and \textit{Generic Private Learner}.

\newtheorem*{new.lemma:stable}{Lemma \ref{lemma:stable}}
\begin{new.lemma:stable}[Stable Histogram, restated]
Let $X$ be any data domain. For $n \ge O(\frac{\log (1/\eta\beta\delta)}{\eta \epsilon})$, 
there exists an $(\epsilon, \delta)$-DP algorithm \textsc{Hist} which with probability 
at least $1-\beta$, on input $S = (x_1, \cdots, x_n)$ outputs a list $L \in X$ and 
a sequence of estimates $a \in [0,1]^{|L|}$ such that 
\begin{enumerate}[leftmargin = *]
\item Every $x$ with $\textup{Freq}_{S}(x) \ge \eta$ appears in $L$, and
\item For every $x \in L$, the estimate $a_x$ satisfies $|a_x - \textup{Freq}_{S}(x)| \le \eta$,
\end{enumerate}
where $\textup{Freq}_{S}(x) = \big|\{i \in [n] ~|~ x_i = x \}\big|/n$.
\end{new.lemma:stable}

\begin{lemma}[Generic Private Learner, \citep{bun2020equivalence}]\label{lemma:genericprivate}
Let $\cH \subset [K]^{\cX}$ be a collection of multi-class hypotheses. For $n = O(\frac{\log |\cH| + \log (1/\beta)}{\alpha \epsilon})$, there exists an $(\epsilon, 0)$-DP algorithm \textsc{GenericLearner} : $(\cX \times [K])^n \rightarrow \cH$ satisfying the following;
let $\cD$ be a distribution over $\cX \times [K]$ such that there exists an $h^{\star} \in \cH$ with $\text{loss}_{\cD} (h^{\star}) \le \alpha$. Then on input ${S} \sim \cD^n$, \textsc{GenericLearner} outputs, with probability at least $1-\beta$, a hypothesis $\hat{h} \in \cH$ such that $\text{loss}_{S} (\hat{h}) \le 2\alpha$.
\end{lemma}

Now we are ready to prove Theorem \ref{thm:GSPL}.

\begin{proof}[Proof of Theorem \ref{thm:GSPL}]

The learning algorithm $M$ is built on top of the Stable Historgram and the Generic Private Learner as described in Algorithm \ref{alg:DPlearner}. According to Lemma \ref{lemma:stable} and \ref{lemma:genericprivate}, we choose parameters
\[
k = O\big( \frac{ \log (1/\eta \beta\delta)}{\eta \epsilon} \big), \quad n^{\prime} = O\big( \frac{ \log (1/\eta \beta)}{\alpha \epsilon} \big).
\]

\begin{algorithm}[h]
    \caption{Differentially-Private Learner $M$}
    \begin{algorithmic}[1]\label{alg:DPlearner}
        \STATE Let $S_1, \cdots, S_k$ each consist of i.i.d. samples of size $m$ from $\cD$. Run $G$ on each batch of samples producing $h_1 = G(S_1), \cdots, h_k = G(S_k)$
        \STATE Run the Stable Histogram algorithm \textsc{Hist} on input $H = (h_1,\cdots,h_k)$ using privacy $(\epsilon/2,\delta)$ and accuracy $(\eta/8, \beta/3)$, publishing a list $L$ of frequent hypotheses
        \STATE Let $S^{\prime}$ consist of $n^{\prime}$ i.i.d. samples from $\cD$. Run \textsc{GenericLearner}$(S^{\prime})$ using $L$ with privacy $\epsilon/2$ and accuracy $(\alpha/2, \beta/3)$ to output a hypothesis $\hat{h}$
    \end{algorithmic}
\end{algorithm}

We show that the algorithm $M$ is $(\epsilon,\delta)$-DP. During the executions of $G(S_1), \cdots G(S_k)$, a change to one entry in a certain $S_i$ changes at most one outcome $h_i \in H$. Thus, differential privacy for this step is observed by taking expectations over the coin tosses of all the executions of $G$. Then the differential privacy for overall algorithm holds by simple composition of differentially-private \textsc{Hist} and \textsc{GenericLearner}.

Next, we prove that the algorithm $M$ is accurate. By standard generalization arguments, we have with probability at least $1-\beta/3$,
\[
\big|\textup{Freq}_H(h) - \prob_{S \sim\cD^m}\big(G(S)= h\big) \big|\le \frac{\eta}{8}
\]
for every $h \in  [K]^{\cX}$ as long as $k \ge O(\log (1/\beta)/\eta)$. Conditioned on this event, by accuracy of \textsc{Hist}, with probability $1-\beta/2$, it produces a list $L$ containing $h^{\star}$ together with a sequence of estimates that are accurate to within an additive error $\eta /8$. Then, $h^{\star}$ appears in $L$ with an estimate $a_{h^{\star}} \ge \eta - \eta/8 -\eta/8 = 3\eta/{4}$.

Now remove from $L$ every item $h$ with $a_h \le \frac{3\eta}{4}$. Since every estimate is accurate within $\eta/8$, $h$ appears in $L$ such that $\textup{Freq}_H(h) \ge \frac{3\eta}{4} - \frac{\eta}{8}= \frac{5\eta}{8}$. Since sum of frequencies is less than 1, the number of list $L$ should be less than $2/\eta$ (i.e. $|L| \le 2/\eta$). This list contains $h^{\star}$ such that $\text{loss}_{\cD}(h^{\star}) \le \alpha$. Hence the \textsc{GenericLearner} identifies $h^{\star}$ with $\text{loss}_{\cD}(h^{\star}) \le \alpha/2$ with probability at least $1-\beta/3$.
\end{proof}

\subsection{Extension to the Agnostic setting} \label{app:agnostic}

Theorem \ref{thm:OLPL} showed that online MC learnability continues to imply private MC learnability in the realizable setting. 
A similar result also holds even when the realizability assumption is violated, which is called \textit{agnostic setting}.

\begin{corollary}[Agnostic setting : Online MC learning implies private MC learning] \label{coro:OLPL_agnostic}
Let $\cH \subset [K]^{\cX}$ be a MC hypothesis class with $\Ldim(\cH) = d$.
Let $\epsilon, \delta \in (0,1)$ 
be privacy parameters and let $\alpha, \beta \in (0,1/2)$ be accuracy parameters. For
$n = O_d \big( \frac{\log (1/\beta \delta)}{\alpha^2 \epsilon} \big)$, 
there exists $(\epsilon, \delta)$-DP learning algorithm such that for every distribution $\cD$, 
given an input sample $S \sim \cD^n$, the output hypothesis $f = \mathcal{A}(S)$ satisfies
$$ 
\text{loss}_{\cD}(f) \le \min_{h \in \cH}\text{loss}_{\cD}(h) + \alpha
$$ with probability at least $1-\beta$.
\end{corollary}

\begin{proof}	
\citet[Theorem 6]{pmlr-v125-alon20a} propose an algorithm, $\cA_{PrivateAgnostic}$, which transforms a private learner in the realizable setting to a private learner that can operate in the agnostic setting. 
The main idea is based on the standard sub-sampling method, and as a result, the transformed agnostic learner has a larger sample complexity by a factor of $1/\epsilon$.
Then Corollary \ref{coro:OLPL_agnostic} is shown by applying $\cA_{PrivateAgnostic}$ to the realizable learner used in Theorem \ref{thm:OLPL}.
\end{proof}


\subsection{Proof of Theorem \ref{thm:sufficient}} \label{app:suff}

We complete the proof of Theorem \ref{thm:sufficient}. 
The proof for Condition 4 is given in the main body.

\newtheorem*{new.thm:sufficient}{Theorem \ref{thm:sufficient}}

\begin{new.thm:sufficient}[restated]
Let $\cF \subset \cY^{\cX}$ be a real-valued function class such that
$\fat_{\gamma}(\cF) < \infty$ for every $\gamma>0$. If one of the following conditions holds, then $\cF$ is privately learnable.
\begin{enumerate}[leftmargin = *]
\item Either $\cF$ or $\cX$ is finite.
\item The range of $\cF$ over $\cX$ is finite (i.e., $\big|\{ f(x) ~|~ f \in \cF, x \in \cX \}\big| < \infty$).
\item $\cF$ has a finite cover with respect to the sup-norm at every scale.
\item $\cF$ has a finite sequential Pollard Pseudo-dimension.
\end{enumerate}
\end{new.thm:sufficient}

\begin{proof}
1. If $|\cF| < \infty$, then for sample complexity $n = \cO(\frac{\log |\cF| + \log(1/\beta)}{\alpha\epsilon})$ 
we directly run the $\epsilon$-DP Generic Private Learner to output with probability at least $1-\beta$, 
a hypothesis $\hat{f} \in \cF$ such that $\text{loss}_{S}(\hat{f}) \le \alpha$. 
Next, assume that $\cX$ is finite. The finiteness of $\cX$ does not imply finite $|\cF|$ because 
$\cY$ is continuous, but we can discretize $\cF$ at some scale $\gamma$, 
which gives us a finite MC hypothesis class $[\cF]_{\gamma}$. 
It is private-learnable by $\epsilon$-DP Generic Private Learner, and then the original class $\cF$ is 
also privately-learnable within accuracy $\gamma$.


2. Observe that this regression problem is essentially a MC problem. 
Furthermore, $\Ldim(\cF)$ by considering it as a MC problem is bounded above by $\fat_{\gamma}(\cF)$, where $\gamma$ is the minimal gap between consecutive values in the range of $\cF$ over $\cX$. 
This means that $\Ldim(\cF)$ is finite, and hence by the argument of Section \ref{subsec:OLPL_multiclass}, $\cF$ is privately learnable. 

3. Given an accuracy $\alpha$, $\cF$ has $n$ finite covers with a radius $r < \alpha$. We construct a set of representative function as 
$\cF^{\prime} = \{f_1, \cdots, f_n\} \subset \cF$ by arbitrarily choosing a representative $f_i$ from the $i$-th cover, and then run $\epsilon$-DP 
Generic Private Learner on $\cF^{\prime}$ to output a hypothesis $\hat{f} \in \cF$ with a small population loss.
\end{proof}
